\documentclass[lettersize,journal]{IEEEtran}
\usepackage{amsmath,amsfonts,amsthm}
\usepackage{amssymb}
\usepackage{algorithmic}
\usepackage{algorithm}
\usepackage{array}
\usepackage{textcomp}
\usepackage{stfloats}
\usepackage{url}
\usepackage{verbatim}
\usepackage{graphicx}
\usepackage{cite}
\usepackage{enumitem}
\usepackage{booktabs}
\usepackage{xcolor}
\usepackage{threeparttable}
\usepackage{float}

\usepackage{tabularx} 
\usepackage{multirow}  
\usepackage{booktabs}  
\usepackage{graphicx}  
\usepackage{array}    
\usepackage{siunitx}   

\usepackage{caption}
\usepackage{subcaption}

\usepackage{makecell}

\newtheorem{theorem}{Theorem}

\usepackage{threeparttable}

\begin{document}

\title{Characterisation of Density-based FM generation methods in the context of Information Fusion}

\author{
\IEEEauthorblockN{Yanhao Huang, Christian Wagner} \\
\IEEEauthorblockA{\textit{Lab for Uncertainty in Data and Decision Making (LUCID)} \\
\textit{School of Computer Science, University of Nottingham}\\
Nottingham, United Kingdom \\
psxyh15@nottingham.ac.uk}
}

\maketitle

\begin{abstract}

Fuzzy Integral (FI) based aggregation provides a powerful mechanism for nuanced aggregation, for example, in ensemble approaches or decision-level fusion more generally. The main challenge of this approach is the appropriate parametrization of the Fuzzy Measure (FM), which captures the worths of the individual components--and their combinations--which are being fused. Here, widely used approaches including the Sugeno-$\lambda$ and Decomposable FMs,  parametrize the FM by extrapolating from the densities, i.e. the weights associated with individual sources, while respecting the FM's monotonicity constraint. This paper articulates that this information is, in general, insufficient to uniquely identify a discrete FM; but shows how an interval-valued FM can indeed be determined uniquely. 
We proceed to show how the incorporation of additional information beyond the above, such as the choice of a specific FI and a dataset, then allows for obtaining even more specific interval-valued FMs. In practice, establishing the quality of an empirically determined FM is not trivial. To help address this, we show how the likelihood with which a resulting interval FM encompasses the `ideal', i.e. the commonly intangible, best, or ground-truth numeric FM, can be determined, producing a confidence interval at a given confidence level. Finally, based on a series of experiments, we demonstrate empirically that the Choquet FI output based on this FM can also be regarded as the confidence interval for the `ideal' information fusion result, providing a novel means to characterize FI fusion outcomes a priori and charting a pathway for future research.

\end{abstract}


\section{Introduction} \label{sec1:introduction}
\IEEEPARstart{A}
measure is a set function that evaluates a set by assigning a number, typically a crisp value, that reflects its size, weight or worth. Among various types of measures, the fuzzy measure (FM) \cite{Sugeno1993FM} is a powerful tool to evaluate the worth of subsets of information sources, 
such as different experts in an expert decision-making context \cite{SIRBILADZE200571FM}; the worth of different sensors, in a sensor fusion context \cite{Gader2004FM}; or the worth of various classifiers, in an ensemble classification context \cite{Subhrajit2021FM}.
In principle, one can envision an `ideal', or \emph{true FM} in the sense that such a FM captures the true worth of all sources (e.g. sensors) and of their combinations. However, in practice, identifying this FM is challenging, not least as datasets are finite and may be biased. Perhaps even more crucially, variability and uncertainty in most problems means that such an ideal FM will generally not be a numeric FM, but instead be itself interval or distribution/fuzzy set valued.

Various methods to parametrize the FM have been developed over the years.
One of the most commonly applied methods is to generate FM from the worths of individual sources, i.e. the densities, in combination with 
the monotonicity constraint imposed by the FM, see Section~\ref{Generating the Fuzzy Measure},
including the popular Sugeno-$\lambda$ FM ($S_{\lambda}$-FM) \cite{Sugeno1993FM} and the $\perp$-decomposable FM (DFM)\cite{Siegfried1984FM}. 
As discussed in \cite{Huang2025conference}, FMs generated using this type of method may not capture the best possible FM for the given context, 
as the densities alone are insufficient to capture its underlying information.

When using the FM in combination with a Fuzzy Integral (FI) in an aggregation context, optimization techniques offer an alternative approach to identifying an appropriate FM, including least-squares based approaches 
\cite{Grabisch1995FM}
, minimum variance approaches \cite{KOJADINOVIC2005131FM,KOJADINOVIC2007498FM} and minimum distance approaches \cite{Kojadinovic2007FM}.
In order to apply such approaches, 
a specific fuzzy integral (FI), and a dataset must be selected and available to enable the computation of aggregated outcome and associated loss in respect to a ground truth.
This is also the case when FMs are identified through machine learning techniques more generally \cite{Simth2017FM,Scott2017ChI-DE,Islam2020FI}.

FMs generated by such optimization approaches are, as such, not independent of the FI selected. In other words, they are a step away from the true, or ideal FM mentioned earlier, with a bias specific to the FI chosen \cite{Christian2017FI}. 
For example, if a FM is identified in respect to the Sugeno FI \cite{Sugeno1993FM}, its parameters may be different than if the Choquet FI \cite{Murofushi1989ChoquetFI} had been chosen. 

In this paper, we explore the extent to which an underlying `ideal' and discrete FM can actually be identified based on basic principles and information available, i.e. the worth of the densities and the monotonicity constraint of the FM. 
Unlike aforementioned heuristic approaches such as the Sugeno-$\lambda$ and $\perp$-decomposable FMs which derive a numeric FM, 
we show that the information encoded in these heuristics is insufficient to determine a numeric FM precisely. 
Nevertheless, we highlight that this basic information is sufficient to determine interval bounds--or indeed, an interval-valued FM ($\overline{FM}$). 
We refer to this FM as the $\overline{FM}_{CA}$, where the subscript `$CA$' denotes `Context-Agnostic'. 

Building on the above, we analyse the relationship between traditional approaches to establish the FM including $S_{\lambda}$-FM and DFM, and $\overline{FM}_{CA}$. 
We discuss how these FMs are approximations of the `ideal' FM in the sense that they are independent of any specific FI. 
Conversely, if further, application-specific and thus non-general information is available, such as via the selection of a specific FI and/or a dataset, 
then a narrower subset of the context-agnostic interval-valued FM can be determined.

We propose a Monte Carlo (MC) based approach to extract this narrower $\overline{FM}$ based on the $\overline{FM}_{CA}$ in respect to a specific FI and a dataset. 
A threshold is applied to select a set of \emph{best} FMs sampled by MC to construct this. 
We refer to the resulting FM as $\overline{FM}_{CS}$ where the subscript `$CS$' denotes `Context-Specific':
it reflects the best-possible estimate of a context-specific FM based on available information.  
It can be implied that the `ideal' discrete FM is more likely to be located in this FM. 
But, as the `ideal' FM is unobtainable in practice, we in general do not know the likelihood of whether $\overline{FM}_{CS}$ actually encompasses the `ideal' FM. 

Thus, to describe the likelihood of $\overline{FM}_{CS}$ encompassing the `ideal' FM, a confidence level ($1-\alpha$) is introduced to transform the $\overline{FM}_{CS}$ into a new FM that acts as the confidence interval in respect to the `ideal' FM. We refer to the resulting FM as $\overline{FM}_{CI}$ where the subscript `$CI$' denotes `Confidence Interval'. We empirically show that the $\overline{FM}_{CI}$ covers the `ideal' FM at a given confidence level. 

Finally, further exploration on fusing information using the Choquet FI based on this FM shows that the FI outputs act as the confidence interval for the `ideal' information fusion result with the confidence level established for the FM. This is significant step forward as it enables the confidence-characterisation of the aggregates produced via the FI a priori--for the fist time.

Fig.~\ref{introduction_chart} presents an overview of the FMs proposed in this paper, along with the information they use, examples of worths corresponding to the source
combination $\{x_{1},x_{2}\}$ for illustration, and the Sections in which they are introduced. We re-emphasize that while one can theoretically envision the existence of the `ideal' FM shown for reference, it is unobtainable in practice. The purpose of the proposed FMs are to better approximate this `ideal' FM based on available information. We also note that of course such as `ideal' FM can also be conceptualized as an interval, for example under noise. Due to space limitations, we do not elaborate on this further here.

\begin{figure}[htbp]
    \centering
    \includegraphics[width=0.45\textwidth]{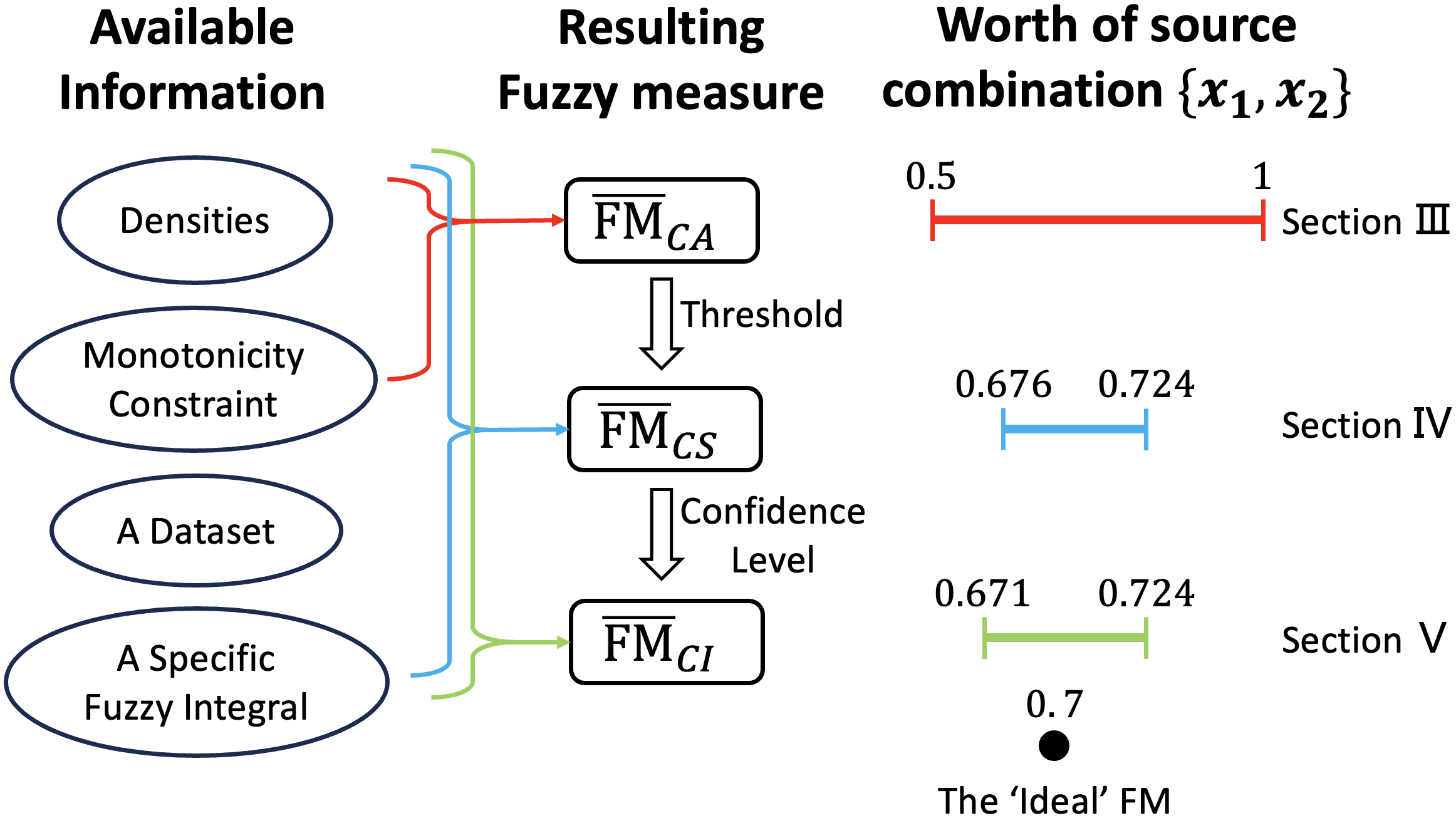}
    \caption{The proposed FM and their corresponding information.}
    \label{introduction_chart}
\end{figure}

The contributions of this paper are as follows.
\begin{enumerate}
    \item We articulate that the monotonicity constraint combined with discrete densities provides only sufficient information to inform an interval-valued FM, which we refer to as $\overline{FM}_{CA}$. We show that traditional FM generation techniques place FMs within the $\overline{FM}_{CA}$ but sometimes in quite different places, shedding light on their behaviour.

    \item For contexts where additional information is available, specifically where a specific FI and dataset are selected, we propose a MC based approach to construct a narrower $\overline{FM}_{CS}$. It reflects the best-possible interval-valued FM estimate based on available information.  

    \item In order to describe the likelihood of $\overline{FM}_{CS}$ covering the `ideal' discrete FM, we show how a confidence level can be assigned to transform it into a context-specific confidence interval-valued $\overline{FM}_{CI}$. This can be regarded as the confidence interval for the `ideal' discrete FM. 
    
    \item Finally, We empirically show that the interval-valued Choquet FI output generated in respect to $\overline{FM}_{CI}$ can be interpreted as the confidence interval for the `ideal' information fusion result with the confidence level of the $\overline{FM}_{CI}$, providing the first such a priori characterisation to the best of our knowledge.

\end{enumerate}

The remainder of the paper is organized as follows. Section \ref{sec2:background knowledge} provides the background knowledge of this paper, including brief introductions about FM and the FI. 
Section~\ref{Generate interval-valued FM from densities}, Section~\ref{Interval-valued FM based on a specific FI and a dataset} and Section~\ref{Generating Context-Specific Confidence Interval-valued FMs} provide step by step introductions on the generation of the $\overline{FM}_{CA}$, $\overline{FM}_{CS}$ and $\overline{FM}_{CI}$, respectively. In Section \ref{Experiment}, we use an example to demonstrate the proposed method. 
Finally, Section \ref{Conclusion} concludes the paper and presents future work. 
To assist the reader with the numerous acronyms and notation, we have compiled a selected list in Table~\ref{Acronyms and notation}.

\begin{table}[h]
    \centering
    \caption{Acronyms and notation}
    \begin{tabular}{cccc}
        \toprule
        FI & fuzzy integral & FM & fuzzy measure \\ 
        $\overline{FM}$ & interval-valued FM & MC & Monte Carlo \\
        CFI & the Choquet FI & CI & confidence interval \\
        $\overline{FM}_{CA}$ & \multicolumn{3}{c}{context-agnostic interval-valued FM} \\
        $\overline{FM}_{CS}$ & \multicolumn{3}{c}{context-specific interval-valued FM} \\
        $\overline{FM}_{CI}$ & \multicolumn{3}{c}{context-specific confidence interval-valued FM} \\
        \midrule
        \multicolumn{1}{c}{$X=\{x_{i}\}$} &  \multicolumn{3}{l}{set of sources} \\
        
        \multicolumn{1}{c}{$s$} &  \multicolumn{3}{l}{Monte Carlo sample size} \\
        \multicolumn{1}{c}{$l$} &  \multicolumn{3}{l}{threshold for selecting the top-performing discrete FMs} \\
        \multicolumn{1}{c}{$\alpha$} &  \multicolumn{3}{l}{significant level of the $\overline{FM}_{CI}$} \\
        
        \multicolumn{1}{c}{$A_{k}$} & \multicolumn{3}{l}{set of all possible permeations of $k$ sources} \\

        \multicolumn{1}{c}{$h_{x_i,t}$} & \multicolumn{3}{l}{evidence provided by source $x_i$ at time period $t$} \\
        \multicolumn{1}{c}{$Y=\{y_{t}\}$} & \multicolumn{3}{l}{set of ground truth at time period $t$} \\
        \multicolumn{1}{c}{$g^{MC}_{s}$} & \multicolumn{3}{l}{set of discrete FMs sampled by MC with the size of $s$} \\

        \multicolumn{1}{c}{$\boldsymbol{g}_{l,s}$} & \multicolumn{3}{l}{set of top-performing discrete FMs} \\

        \multicolumn{1}{c}{$L_{l,s}$} & \multicolumn{3}{l}{lower bound of the $\overline{FM}_{CS}$} \\
        \multicolumn{1}{c}{$U_{l,s}$} & \multicolumn{3}{l}{upper bound of the $\overline{FM}_{CS}$} \\
        \multicolumn{1}{c}{$\bar{L}_{\alpha,l,s}$} & \multicolumn{3}{l}{confidence interval for the lower bound of the $\overline{FM}_{CS}$} \\
        \multicolumn{1}{c}{$\bar{U}_{\alpha,l,s}$} & \multicolumn{3}{l}{confidence interval for the upper bound of the $\overline{FM}_{CS}$} \\

        \multicolumn{1}{c}{$\bar{g}_{CA}$} &  \multicolumn{3}{l}{notation for the $\overline{FM}_{CA}$} \\
        \multicolumn{1}{c}{$\bar{g}_{l,s}$} &  \multicolumn{3}{l}{notation for the $\overline{FM}_{CS}$} \\
        \multicolumn{1}{c}{$\tilde{g}_{\alpha,l,s}$} &  \multicolumn{3}{l}{notation for the $\overline{FM}_{CI}$} \\
        \bottomrule
    \end{tabular}
    \label{Acronyms and notation}
\end{table}

\section{Background Knowledge}\label{sec2:background knowledge}

\subsection{Fuzzy Measure}
In a finite context which is the norm in practical applications, the fuzzy measure (FM) $g$ is a set function which maps the power set of $X$ to the given weights, i.e. $g:2^{X}\rightarrow [0,1]$, and follows these constraints: 
\begin{enumerate}
    \item Boundary conditions:\\
    $g(\varnothing)=0$ and $g(X)=1$
    \item Monotonicity:\\
    If $A\subseteq B\subseteq X$, then $g(A)\leq g(B)$
\end{enumerate}
Where $X=\{x_1,x_2,\cdots,x_{I}\}$ is a nonempty finite set and $x_{i}\in X$, $i=\{1,\cdots,I\}$ represent the $i^{th}$ information source. 

\subsubsection{Generating the Fuzzy Measure} \label{Generating the Fuzzy Measure}

One of the most popular approaches for parametrizing an FM is the Sugeno-$\lambda$ FM $(S_{\lambda}$-FM) \cite{Sugeno1993FM} :
\begin{equation} \label{sec2:sugeno measure}
g(A\cup B)=g(A) + g(B) + \lambda g(A)g(B) \text{,}
\end{equation}
where $A,B\subseteq X$, $A\cap B=\varnothing$, and $\lambda>-1$. Here, $\lambda$ can be obtained by solving

\begin{equation}\label{sec2:lambda solver}
\lambda + 1 = \prod_{i=1}^I(1+\lambda g(\{x_{i}\}))  \text{,}
\end{equation}
where $\lambda \in (-1,+\infty),\lambda \neq 0$.

In addition, a simpler approach to extract the FM is the $\perp$-decomposable FM (DFM) \cite{Siegfried1984FM} :
\begin{equation}
g(A\cup B) = g(A) \perp g(B) \text{,}
\end{equation}
where $A,B\subseteq X$, $A\cap B=\varnothing$ and $\perp$ represents a t-conorm. In this paper, the bounded sum \cite{Erich2005tconorm} 
\begin{equation} \label{decomposed measure min}
g(A\cup B) = min(g(A)+g(B),1) 
\end{equation}
and the maximum \cite{Erich2005tconorm}
\begin{equation} \label{decomposed measure max}
g(A\cup B) = max(g(A),g(B)) 
\end{equation}
are used as the t-conorm. We note them as $g_{DM,bs}$ and $g_{DM,max}$ respectively.   
As a reminder, the above two FM generation methods only use the densities and monotonicity constraint to identify discrete FMs. 

\subsubsection{Interval-valued FM}

In practice, an `ideal' discrete FM may not exist for a number of reasons, from noise in the available data to context-specific variability in the worths. Here, interval ($\overline{FM}$) may offer a more comprehensive representation of the underlying weighting structure.
Let $\bar{g}$ be an $\overline{FM}$, the monotonicity constraint of the $\overline{FM}$ referring to \cite{Caimei1995FM} is as follows:
\begin{align} \label{IFM_MC}
If\  A\subseteq B\subseteq X,\ then\  \bar{g}(A)\leq \bar{g}(B), \tilde{g}(A) \leq \tilde{g}(B).
\end{align}
In this constraint, $\bar{g}(A)$ and $\bar{g}(B)$ are two interval-valued numbers. 
According to \cite{Caimei1995FM}, 
\begin{align} \label{interval_larger}
\bar{g}(A)\leq \bar{g}(B)\ \ iff \ \ g^{-}(A) \leq g^{-}(B), g^{+}(A) \leq g^{+}(B),
\end{align}
where $g^{-}(A)$ represents the lower bound of $\bar{g}(A)$ and $g^{+}(A)$ represents the upper bound respectively.

\subsection{Fuzzy Integral}

The Choquet Fuzzy Integral (CFI) \cite{Murofushi1989ChoquetFI}
is one of the most common FI methods, fusing evidence in respect to the given FM $g$. Let $h:X\rightarrow R$ be the integrand, the CFI is defined as follows:

\begin{equation}
\int_{C}h\circ g=C_{g}(h)=\sum_{i=1}^{I}h(x_{\pi(i)})
(g(B_{i})-g(B_{i-1}))\text{,}
\label{Choquet FI}
\end{equation}
where $\pi$ is a permutation of $X$, such that
$h(x_{\pi(1)})\geq h(x_{\pi(2)})\geq \cdots \geq h(x_{\pi(I)})$,
$B_{i}=\{x_{\pi(1)},\cdots,x_{\pi(i)}\}$ and
$g(B_{0})=0$. 

Guo et al.\cite{Caimei1995FM} also introduced the concept of FI of interval-valued evidence in respect to the $\overline{FM}$. 
This paper only focuses on FI of discrete evidence based on the $\overline{FM}$. We treat discrete functions as the special case of interval-valued functions. The CFI of discrete functions in respect to the $\overline{FM}$ referring to \cite{Caimei1995FM} is as follow: 
\begin{align}
\int h\circ \bar{g}=C_{\bar{g}}(h)=[\int h\circ g^{-},\int h\circ g^{+}].
\label{Interval Choquet FI}
\end{align}

\section{Context-Agnostic Interval-valued FMs} \label{Generate interval-valued FM from densities}
Let $g(\{x_{i}\})$ be the worths of individual sources, i.e. the densities. This paper focuses on discrete numeric densities, which will be referred to simply as densities in the remainder of the paper. 

As discussed initially in \cite{Huang2025conference}, densities in combination with the FM's monotonicity constraint are insufficient to uniquely identify a discrete FM, but instead an interval-valued FM can be obtained. 
As this FM is context-agnostic in the sense of not being specified in respect to a specific FI or dataset, we refer to it as $\overline{FM}_{CA}$, using the notation $\bar{g}_{CA}$.  

\subsection{Deriving context-agnostic, interval-valued FMs}

Let $k$ reflect the arity of the given set-valued element within the FM and $A_{k}$ be the set that contains all possible permutations of $k$ sources.
For example, if $k=2$, then $A_{2}=\{\{x_{1},x_{2}\},\{x_{1},x_{3}\},\{x_{2},x_{3}\}\}$. $\overline{FM}_{CA}$ can be obtained as

\begin{align} \label{IFM_DM}
\bar{g}_{CA}(A) = 
\begin{cases}
\bigl[g(A),g(A)\bigr], & k=1 \\
\bigl[max_{x_{i}\in A}(g(\{x_{i}\})),1\bigr], & 1<k<I \\
\bigl[1,1\bigr], & k=I
\end{cases}, 
\end{align}
where $A\in A_{k}$, which means that $A$ is one of the combinations of $k$ sources. Note that, if $A\in A_{1}$, $g(A)$ is the same as $max_{x_{i}\in A}(g(\{x_{i}\}))$ and we simplify $[g(A),g(A)]$ to $g(A)$ and $[1,1]$ to $1$ throughout this paper.

\begin{theorem} \label{theorem:FM_CA_monotonocity}
$\overline{FM}_{CA}$ follows the monotonicity constraint.
\end{theorem}

\subsection{Sampling context-agnostic, discrete FMs}
Fig.~\ref{All possible FMs lattice for three sources} shows the proposed interval-valued FM lattice for three sources, where $g_{i}$ stand for $g(\{x_{i}\})$. 
Sampling from each interval within this FM produces a discrete FM. Note that the only information used to inform $\overline{FM}_{CA}$ is the densities and monotonicity constraint of the FM, and that conversely, this information in of itself provides only limited constraints on possible discrete FM values.

\begin{figure}[h]
    \centering
    \includegraphics[width=0.48\textwidth]{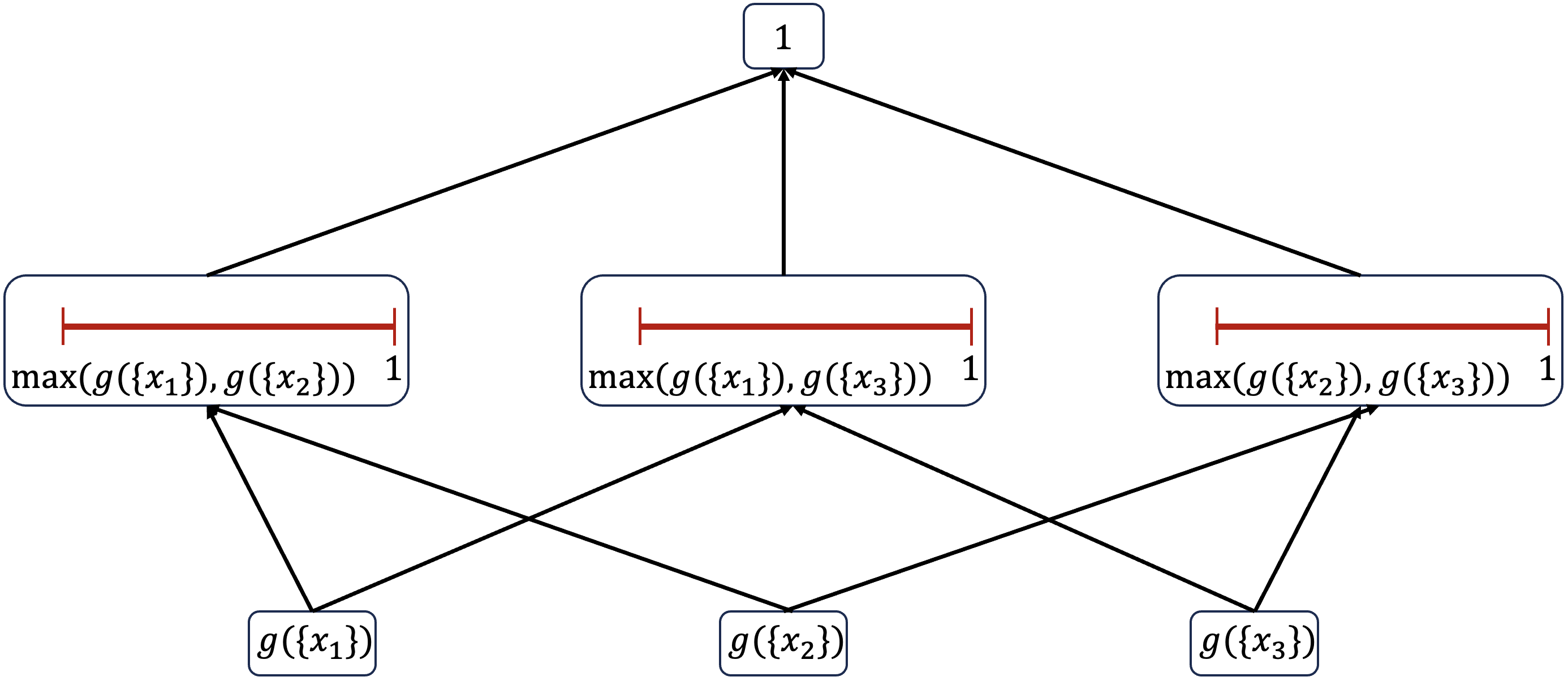}
    \caption{$\overline{FM}_{CA}$ that captures all possible discrete FMs generated from discrete densities for three sources.}
    \label{All possible FMs lattice for three sources}
\end{figure}

However, monotonicity may break depending on the values chosen. For example, in a FM lattice for four sources, assume that $\bar{g}(\{x_{1},x_{2}\})$ is $[0.7,1]$ and $\bar{g}(\{x_{1},x_{2},x_{3}\})$ is $[0.8,1]$. Although $\bar{g}(\{x_{1},x_{2}\}) \leq \bar{g}(\{x_{1},x_{2},x_{3}\})$, if we pick $g(\{x_{1},x_{2}\}) = 0.9 \in[0.7,1]$, $g(\{x_{1},x_{2},x_{3}\})=0.85 \in [0.8,1]$ to generate a discrete FM, the resulting FM does not follow the monotonicity constraints. 

To address this, while sampling discrete FMs, the lower bound of each interval corresponding to $A_{k}$ will be adjusted to the largest worths corresponding to the individual sources , i.e. the densities, in $A_{k}$. It can be obtained as
\begin{align} \label{adjust IFM sampling}
[max_{x_{i}\in \bigcup A_{k}}\bigl(g(\{x_{i}\})\bigr),1],
\end{align}
where $\bigcup A_{k}=\{a|\exists\ x\in A_{k},\ a\in x\}$. Note that, this adjusted interval is only used to sample context-agnostic, discrete FMs, the proposed $\overline{FM}_{CA}$ stays as shown in (\ref{IFM_DM}). 

Algorithm~\ref{Randomly generate numeric FMs from densities} shows a way to address this issue.

\begin{algorithm}[h]
\caption{Randomly generate numeric FMs from densities}
\begin{algorithmic}[1]
\STATE Input the densities as $g(\{x_{i}\})$
\FOR {$k = 2$ to $I$}
    \STATE Calculate the interval of combination $A$ refers to (\ref{adjust IFM sampling}), where $A\in A_{k}$.
    \STATE Randomly select a value from each intervals and use it as the discrete worth of corresponding combination.
\ENDFOR
\end{algorithmic}
\label{Randomly generate numeric FMs from densities}
\end{algorithm}

\subsection{Variability of traditional heuristic FM generation approaches}
Although FM generation approaches based on densities and monotonicity constraint place discrete FMs inside the $\overline{FM}_{CA}$, they are sometimes at quite different positions. 
To illustrate this phenomenon, we construct a synthetic example that also serves to introduce the methods proposed in the remainder of this paper. Here, we only present the synthetic densities, while the details of this example are introduced in the next section.

In a three-source fusion scenario, consider the densities $g(\{x_{1}\})=0.4$, $g(\{x_{2}\})=0.5$ and $g(\{x_{3}\})=0.6$. The worths in the second layer of the $\overline{FM}_{CA}$ lattice are $g(\{x_{1},x_{2}\})=[0.5,1]$, $g(\{x_{1},x_{3}\})=[0.6,1]$ and $g(\{x_{2},x_{3}\})=[0.6,1]$. The $S_{\lambda}$-FM, $g_{DM,bs}$ and $g_{DM,max}$ can be calculated according to (\ref{sec2:sugeno measure}), (\ref{decomposed measure min}) and (\ref{decomposed measure max}). Fig.~\ref{The locations of three traditional FMs within the IFM-DM} shows that these three traditional FM generation approaches place FMs inside the $\overline{FM}_{CA}$ but at quite different locations.

\begin{figure}[h]
    \centering
    \includegraphics[width=0.45\textwidth]{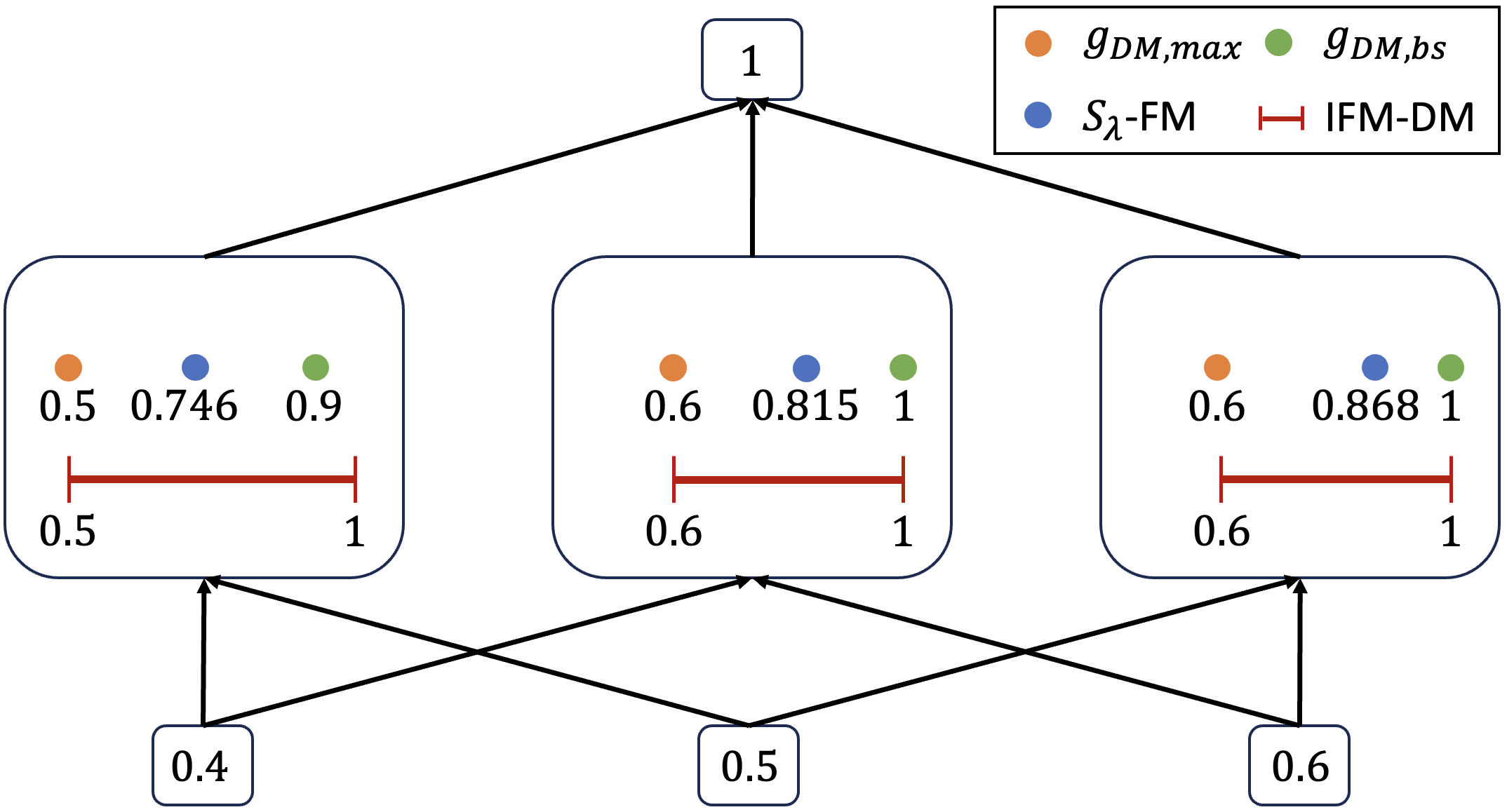}
    \caption{The locations of three traditional FMs within the $\overline{FM}_{CA}$.}
    \label{The locations of three traditional FMs within the IFM-DM}
\end{figure}

These four FMs are all context agnostic, i.e. independent of the choice of a specific FI.
When facing a specific task, once the FI and the dataset are determined, the intervals within the $\overline{FM}_{CA}$ can be further narrowed as discussed in the next Section. 
Since these three traditional FMs are within, but spread across the $\overline{FM}_{CA}$, we continue to explore whether they are \emph{also} within the narrower $\overline{FM}$. 

\section{Generating Context-Specific Interval-valued FMs} \label{Interval-valued FM based on a specific FI and a dataset}

In the previous section, we discussed how the limited information encoded within the densities and the FM's monotonicity constraint can inform associated general, but context-agnostic interval-valued FMs ($\overline{FM}_{CA}$). While providing valuable insight, including by contextualizing traditional discrete FMs, the resulting intervals were very wide, with limited real-world utility. In this section, we consider leveraging additional information to generate so-called context-aware FMs. Specifically, we propose a MC based approach to further narrow the original $\overline{FM}_{CA}$.

We first use MC to sample discrete FMs from $\overline{FM}_{CA}$, with each of the resulting FMs being context-agnostic in of itself.
Subsequently, given a specific context defined by a specifically selected FI and a dataset, we calculate the FI outputs in respect to these FMs and evaluate their performance in respect to the ground truth, e.g. label, of the dataset. 

A trivial approach to selecting the 'ideal' FM would be to directly select the FM with the best performance from the MC sample, but since both the dataset and the MC sample size are finite or perhaps the underlying `ideal' discrete FM itself contains inherent uncertainty, it is unlikely that the best possible FM for the given context was indeed identified through the MC process. 
To account for this, we adopt a performance threshold according to which FMs from the MC process which result in top performance (e.g. within the best $0.1\%$) within the given context, are selected.
We then use the spread of these discrete FMs to derive an interval-valued FM ($\overline{FM}_{CS}$) to cover them.

The remainder of this section provides a step by step introduction of the construction of $\overline{FM}_{CS}$ alongside with a synthetic example. 
We also explore whether traditional FMs are within $\overline{FM}_{CS}$.

\subsection{Step 1: leveraging a dataset} \label{subsec:Obtain the evidence}

Let $h_{x_i,t}$ 
represent the evidence, provided by source $x_i$, i.e. sensors, at time period $t$, where $i=\{1,\cdots,I\}$, $I$ denotes the number of sources and
$t=\{1,\cdots,T\}$, where $T$ denotes the total number of samples provided. Let $Y=\{y_{t}|y_{1},\cdots,y_{T}\}$ be the ground truth. 
Note that the ground truth can be considered as the overall rating in an evaluation task, overall/target output in a multi-sensors fusion task or label 

For the synthetic example, consider a three-sensor fusion tasks, where $h_{x_i,t}$ denotes the output of three distance sensors and $t$ indicates that the data are collected over time. 
The dataset is generated by grid sampling of three sensor outputs, each uniformly sampled from $[0,10]$ with 10 equally spaced values.
This results in $10\cdot10\cdot10 = 1000$ samples. The ground truth of the FM $g_{true}$, i.e. the `ideal' FM that is assumed to be the actual worth of all sub-combinations of these three sensors, is given in Table~\ref{The Ground truth FM for the synthetic example}. The CFI outputs in respect to the `ideal' FM are treated as the ground truth. 
Note that the `ideal' FM is a synthetic FM used only to generate the ground truth and demonstrate that it lies within the $\overline{FM}_{CS}$. 
In practice, the `ideal' FM is unknown, while the ground truth and evidence are generally known, albeit often subject to uncertainty.

\begin{table}[h]
    \centering
    \caption{The `ideal' FM for the synthetic example}
    \begin{tabular}{ccc}
        \toprule
        \boldmath{$k=1$} & \boldmath{$k=2$} & \boldmath{$k=3$} \\ 
        \midrule
        $x_{1}:0.4$      & $x_{1},x_{2} : 0.7$  & $X : 1$ \\ 
        $x_{2}: 0.5$     & $x_{1},x_{3} : 0.8$  &                                \\ 
        $x_{3}: 0.6$     & $x_{2},x_{3} : 0.9$  &                                \\ 
        \bottomrule
    \end{tabular}
    \label{The Ground truth FM for the synthetic example}
\end{table}

\subsection{Step 2: the role of the densities}

This paper focuses on exploring the potential of FMs generated from densities, the densities should first be determined by certain methods, e.g. expert knowledge, classification accuracy \cite{Tahani1990densities} or interCriteria analysis \cite{Sotirova2017densities}. Following the synthetic example, the densities are $g(\{x_{1}\})=0.4$, $g(\{x_{2}\})=0.5$ and $g(\{x_{3}\})=0.6$ given in Section~\ref{Generate interval-valued FM from densities} and can be considered as expert knowledge.

\subsection{Step 3: identification of the context-agnostic interval-valued FM}

Since the densities are obtained, $\overline{FM}_{CS}$ can be identified as introduced in Section~\ref{Generate interval-valued FM from densities}. 
Table~\ref{The IFM-DM for the synthetic example} shows the $\overline{FM}_{CS}$ for the synthetic example.

\begin{table}[h]
    \centering
    \caption{The $\overline{FM}_{CA}$ for the synthetic example}
    \begin{tabular}{ccc}
        \toprule
        \boldmath{$k=1$} & \boldmath{$k=2$} & \boldmath{$k=3$} \\ 
        \midrule
        $x_{1}:0.4$      & $x_{1},x_{2} : [0.5,1]$  & $X : 1$ \\ 
        $x_{2}: 0.5$     & $x_{1},x_{3} : [0.6,1]$  &                                \\ 
        $x_{3}: 0.6$     & $x_{2},x_{3} : [0.6,1]$  &                                \\ 
        \bottomrule
    \end{tabular}
    \label{The IFM-DM for the synthetic example}
\end{table}

\subsection{Step 4: generating discrete FMs using MC}

Monte Carlo (MC) is applied to sample discrete FMs from $\overline{FM}_{CS}$ following Algorithm~\ref{Randomly generate numeric FMs from densities}. Let $g^{MC}_{s}=\{g^{MC}_{j}|g^{MC}_{1},\cdots,g^{MC}_{s}\}$ be the set of discrete FMs sampled by MC with $s$ sample size, where $j=\{1,\cdots,s\}$. Note that $s$ is a hyperparameter.

Following the synthetic example, $s$ is set as $50000$. As shown in Table~\ref{The IFM-DM for the synthetic example}, the intervals within the $\overline{FM}_{CS}$ are $g(\{x_{1},x_{2}\})=[0.5,1]$, $g(\{x_{1},x_{3}\})=[0.6,1]$ and $g(\{x_{2},x_{3}\})=[0.6,1]$. 
MC randomly picks one value for each worth within the corresponding intervals, for example the first sample in first MC simulation can be $g^{MC}_{1}(\{x_{1},x_{2}\})=0.56$, $g^{MC}_{1}(\{x_{1},x_{3}\})=0.80$ and $g^{MC}_{1}(\{x_{2},x_{3}\})=0.75$. 

\subsection{Step 5: fusing the evidence using the FI}

Let $C_{g^{MC}_{j}}(h_{X,t})$ be the CFI output at time node $t$ in respect to $g^{MC}_{j}$. 
For example, assuming that $h_{x_{1},1} = 1$, $h_{x_{2},1} = 2$ and $h_{x_{3},1} = 3$, the CFI output of $g^{MC}_{1}$ is $C_{g^{MC}_{1}}(h_{X,1}) = 3\cdot0.6 + 2\cdot(0.75-0.6) + 1\cdot(1-0.75)=2.35$. 

\subsection{Step 6: extracting the set of top-performing discrete FMs} \label{subsec:Exploring the distribution of the FMs}

We evaluate the performance of $g^{MC}_{j}$ using the mean absolute error (MAE) between the CFI outputs in respect to the ground truth as

\begin{align} \label{calculate MAE}
MAE_{g^{MC}_{j}} = \frac{\sum_{t=1}^{T}|y_{t} - C_{g^{MC}_{j}}(h_{X,t})|}{T}.
\end{align}

The set of top-performing FMs $\boldsymbol{g}_{l,s}$ are retained in respect to the performance threshold $l$ by selecting the top $l\%$ FMs from $g^{MC}_{s}$. For example, if $s=50000$, $l=0.1$, then $\boldsymbol{g}_{0.1,50000}$ contains 50 discrete FMs which are the top $0.1\%$ performance FMs in $g^{MC}_{50000}$. 

Fig.~\ref{The distribution of all discrete FMs for the synthetic example} gives an example of where the best FMs are located for the synthetic example, based on a threshold $l=0.1$.
For illustrative purposes, the figure visualizes the performance of all discrete `FM' combinations, including those which are not valid FMs as they do not follow the monotonicity constraint (red background).
Note that the distribution of the FM performances indicates a peak, with the `ideal' FM (known in this synthetic case) at the tip of this peak. The top performing FMs retained are also shown and are grouped around the `ideal' FM.
Additionally, this figure only aims to visualize the distribution of all density-based FMs, including those that do not follow the monotonicity constraint. In the proposed approach, there is no need to generate FMs that break the monotonicity constraint.

\begin{figure*}[htbp]
	\centering
	\subfloat[$g(\{x_{1},x_{2}\})$]{\includegraphics[width=.63\columnwidth]{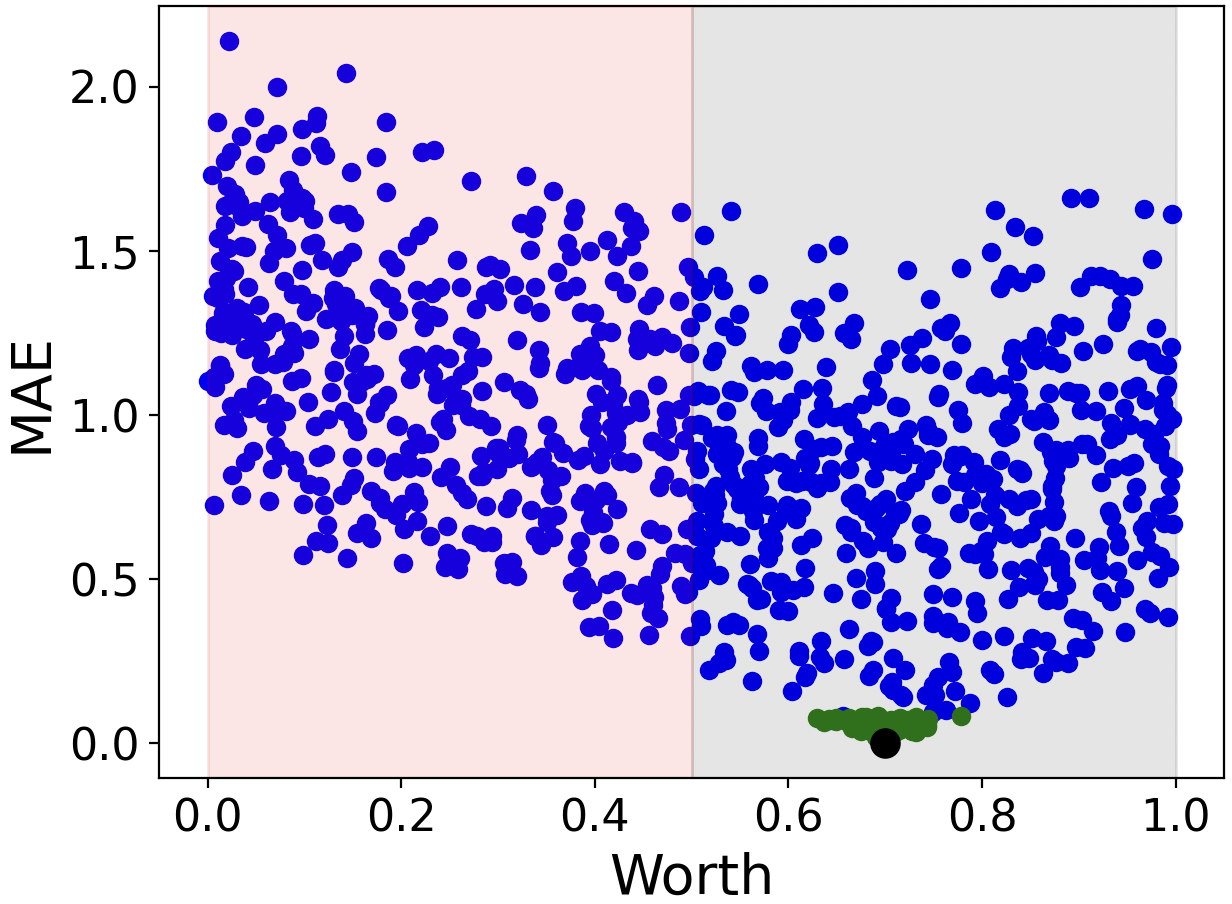}} 
    \hspace{5pt} 
	\subfloat[$g(\{x_{1},x_{3}\})$]{\includegraphics[width=.63\columnwidth]{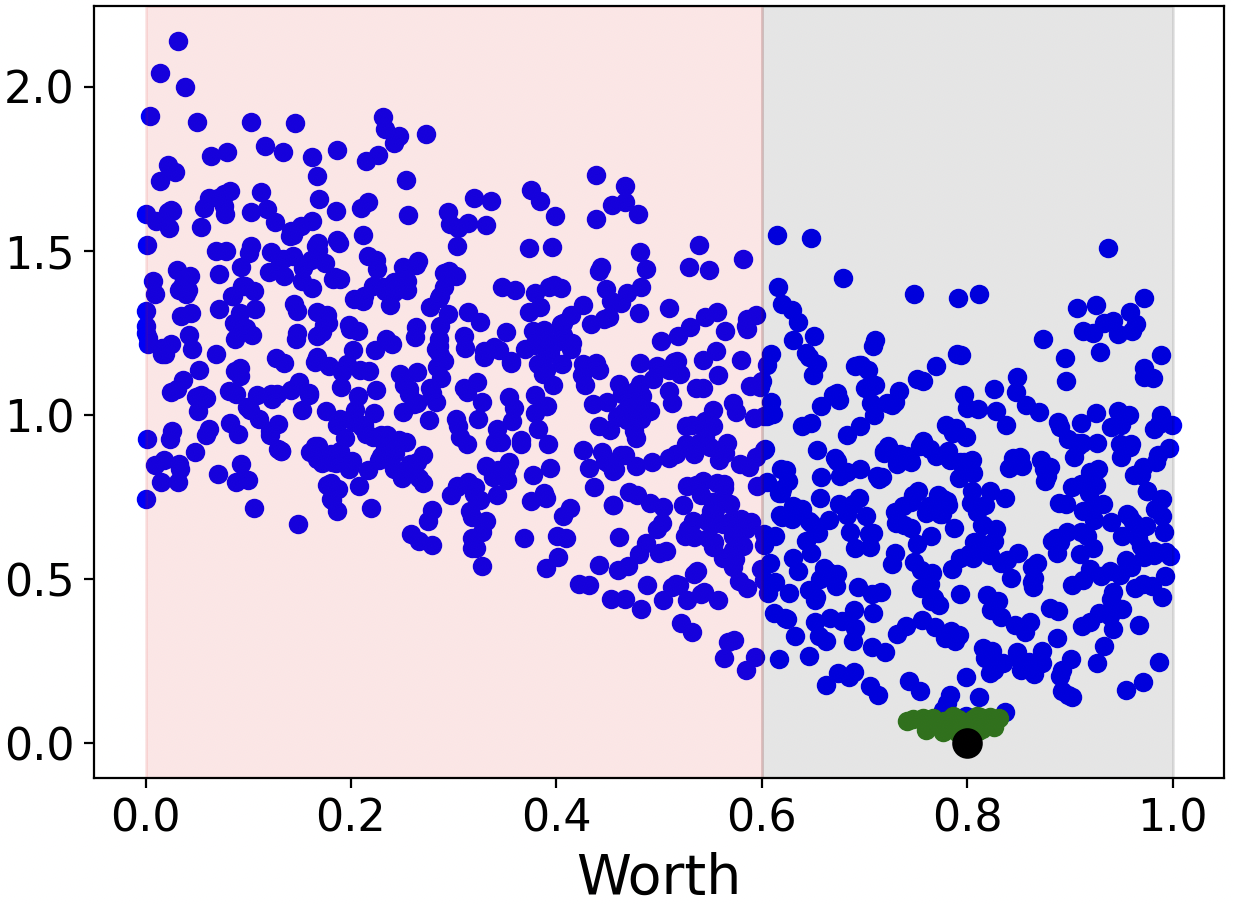}} 
    \hspace{5pt} 
	\subfloat[$g(\{x_{2},x_{3}\})$\label{all_g_0_to_1_g23}]{\includegraphics[width=.63\columnwidth]{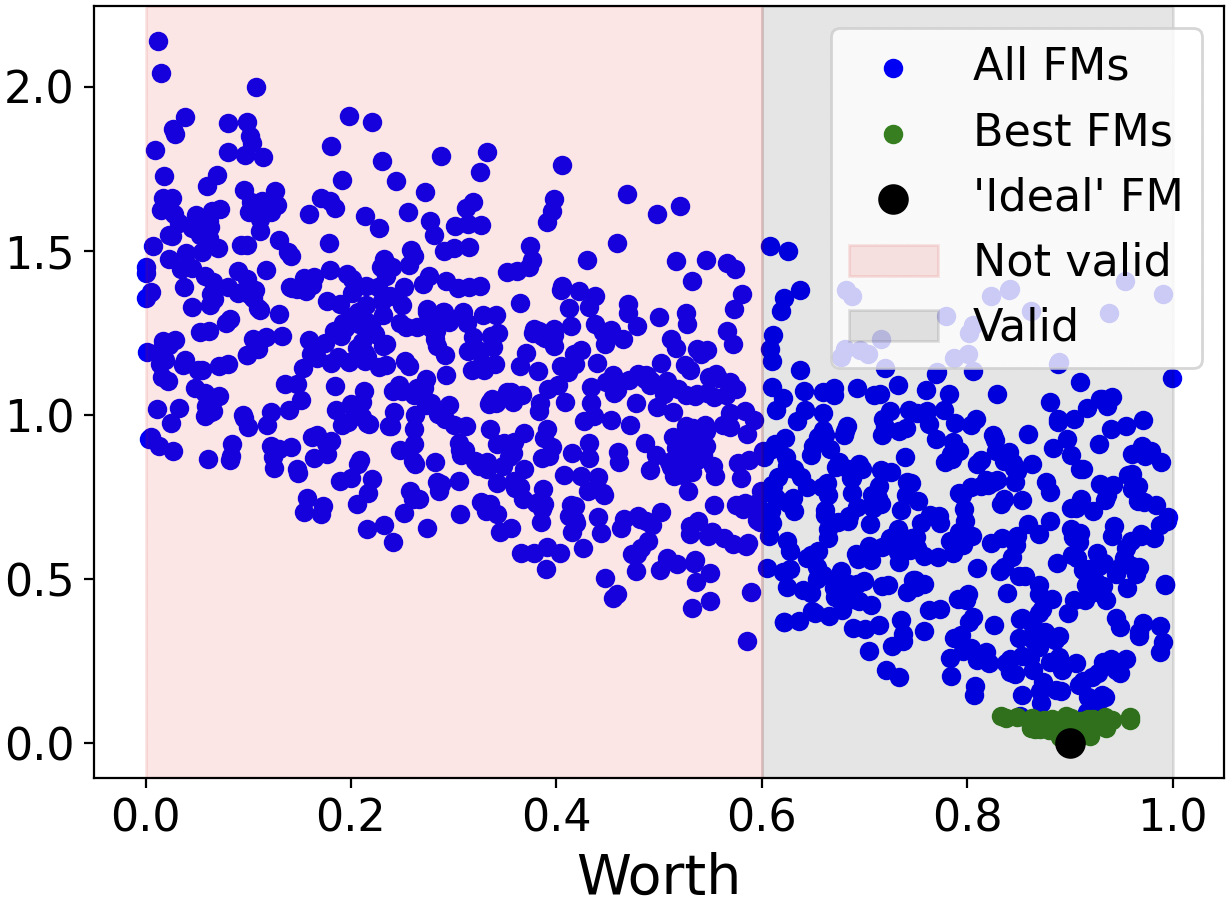}} 
	\caption{The distribution of all discrete FMs for the synthetic example.}
    \label{The distribution of all discrete FMs for the synthetic example}
\end{figure*}

Fig.~\ref{The distribution of $g^{*}_{0.1,50000}$ visualized by histogram and KDE} shows the distribution of $\boldsymbol{g}_{l,s}$ where $l=0.1$ and $s=50000$. The histogram represents the count and the blue curve represents the probability density function estimated by kernel density estimation (KDE). The three vertical lines represent the values according to the label, where `Ground Truth' represents the `ideal' FM in Table~\ref{The Ground truth FM for the synthetic example}, `Mean' represents the basic average of $\boldsymbol{g}_{l,s}$ and `Median' represents the median value of $\boldsymbol{g}_{l,s}$.

\begin{figure*}[htbp]
	\centering
	\subfloat[$g(\{x_{1},x_{2}\})$\label{synthetic_kde_50000/g12}]{\includegraphics[width=.63\columnwidth]{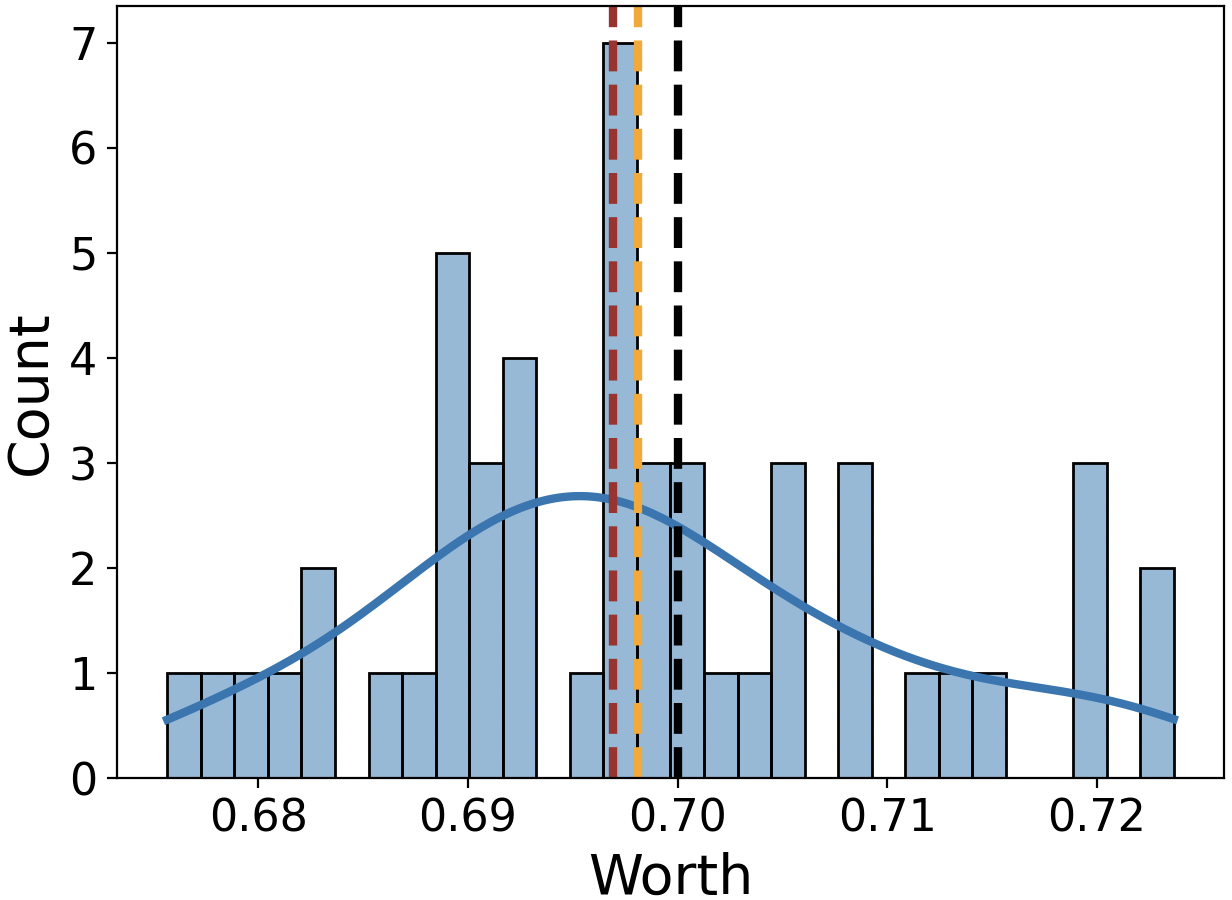}} 
    \hspace{5pt} 
	\subfloat[$g(\{x_{1},x_{3}\})$\label{synthetic_kde_50000/g13}]{\includegraphics[width=.63\columnwidth]{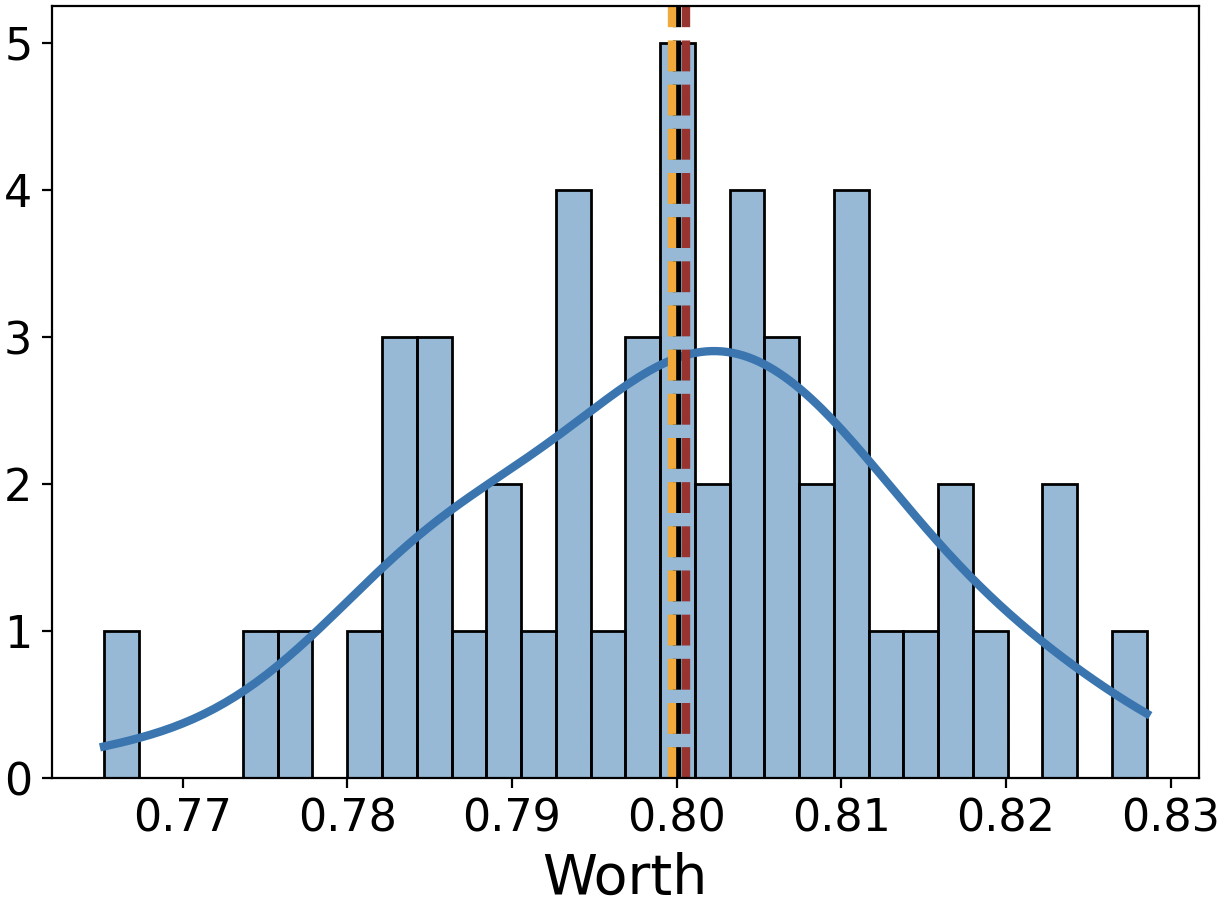}} 
    \hspace{5pt} 
	\subfloat[$g(\{x_{2},x_{3}\})$\label{synthetic_kde_50000/g23}]{\includegraphics[width=.63\columnwidth]{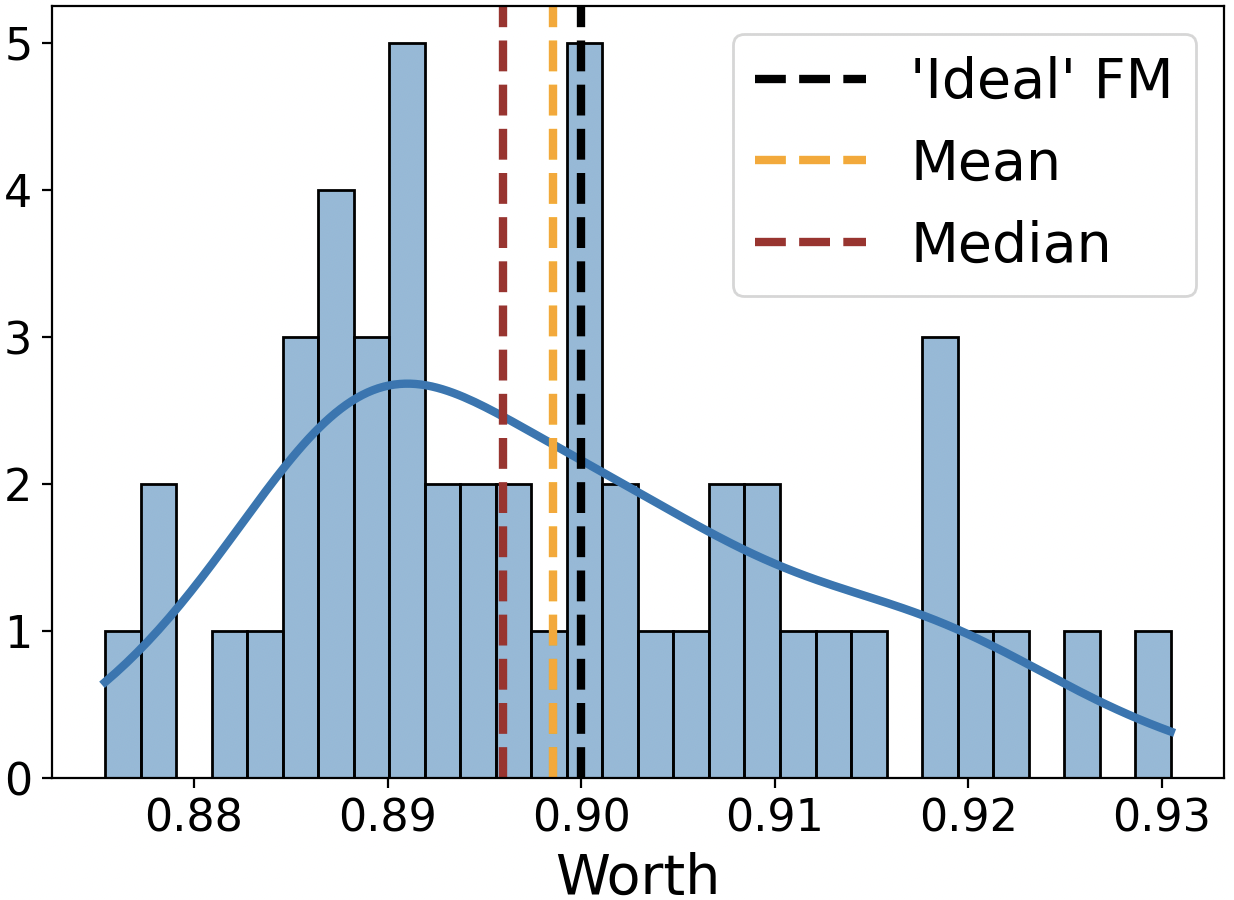}} 
	\caption{The distribution of $\boldsymbol{g}_{0.1,50000}$ visualized by histogram and kernel density estimation.}
    \label{The distribution of $g^{*}_{0.1,50000}$ visualized by histogram and KDE}
\end{figure*}

It can be observed that the maximum point of the blue curve (indicating that samples from this distribution are most likely to lie near this value) is not located at the ground truth. This may be due to the finite MC sample size or the dataset is finite, which can bias the identification of these top-performing FMs.
This phenomenon indicates that it is uncertain where exactly the ground truth lies among these top-performing FMs (which would lead to a discrete FM). However, it is likely that the ground truth lies within the range spanned by these top-performing FMs (which would lead to an $\overline{FM}$). Therefore, rather than a discrete FM, an $\overline{FM}$ is a good representation for data-driven FMs.

Furthermore, the mean point and the median point are not necessarily located at the ground truth, which indicates that to achieve the top-performance FM in this rather simple example, there is still quite substantial scope for individual FM values to trade off against each other. This again underlines that an $\overline{FM}$ is a good representation for data-driven FMs

\subsection{Step 7: generating context-specific interval-valued FM from the top-performing set}

As discussed in the previous subsection, it is likely that the ground truth lies within the range spanned by these top-performing FMs. For each combination of sub-sources, we directly use the minimum value within the corresponding top-performing set as the lower bound of the interval, and the maximum value within this set as the upper bound. 
Note that it is possible to use the $2.5\%$, $5\%$, $25\%$ quantiles to act as the lower bound and the $97.5\%$, $95\%$, $75\%$ quantiles to act as the upper bound, respectively.
Further exploration in using quantiles as the bounds of the interval is left for future work.

Let $\bar{g}_{l,s}$ represent the $\overline{FM}_{CS}$, it can be defined as
\begin{align} 
\bar{g}_{l,s}(A) = 
\begin{cases}
\bigl[g(A),g(A)\bigr], & k=1 \\
\bigl[\min_{g\in\boldsymbol{g}_{l,s}}g(A),\max_{g\in\boldsymbol{g}_{l,s}}g(A) \bigr], & 1<k<I \\
\bigl[1,1\bigr], & k=I
\end{cases}, 
\end{align}
where $A\subseteq X$.
Following the synthetic example, the resulting $\overline{FM}_{CS}$ is shown in Table~\ref{The IFM-CS for the synthetic example}, where $l=0.1$ and $s=50000$.

\begin{table}[h]
    \centering
    \caption{The $\overline{FM}_{CS}$ for the synthetic example}
    \begin{tabular}{ccc}
        \toprule
        \boldmath{$k=1$} & \boldmath{$k=2$} & \boldmath{$k=3$} \\ 
        \midrule
        $x_{1}:0.4$      & $x_{1},x_{2} : [0.676,0.724]$  & $X : 1$ \\ 
        $x_{2}: 0.5$     & $x_{1},x_{3} : [0.765,0.829]$  &                                \\ 
        $x_{3}: 0.6$     & $x_{2},x_{3} : [0.875,0.930]$  &                                \\ 
        \bottomrule
    \end{tabular}
    \label{The IFM-CS for the synthetic example}
\end{table}

\section{Generating Context-Specific Confidence Interval-valued FMs} \label{Generating Context-Specific Confidence Interval-valued FMs}

In the previous section, we introduce the generation of context-specific interval-valued FMs ($\overline{FM}_{CS}$). As the $\overline{FM}_{CS}$ is generated based on MC samples, if we repeat the generation process, the resulting $\overline{FM}_{CS}$ would be different. In addition, we argue that the ground truth, i.e. the `ideal' FM, is more likely to lie within $\overline{FM}_{CS}$, but we do not know \textit{\textbf{how certain}} the `ideal' FM is actually within $\overline{FM}_{CS}$.

Here, we extend the generation of the $\overline{FM}_{CS}$ by adding the significant level $\alpha$ to make the resulting $\overline{FM}$ acts as the confidence interval for the `ideal' FM. We call it the context-specific confidence interval-valued FM ($\overline{FM}_{CI}$). It can be interpreted that if we repeat the generation of the $\overline{FM}_{CI}$ to generate multiple $\overline{FM}_{CI}$, $(1-\alpha)\cdot 100\%$ of them cover the `ideal' FM.

The remainder of this section provides introduction of the construction of $\overline{FM}_{CS}$ alongside with the same synthetic example in the previous section. 
We also provide experiments to test whether the resulting $\overline{FM}_{CS}$ actually covers the `ideal' FM under the given significant level using the synthetic example.

\subsection{Generating confidence intervals for the lower and upper bounds of $\overline{FM}_{CS}$}

As discussed in the previous section, we use the minimum points and the maximum points of the top-performing FM set as the lower bounds and the upper bounds of the intervals in $\overline{FM}_{CS}$. 
Here, the top-performing FMs are obtained by MC, and as the MC sample size is finite, the top-performing FMs would be different under resampling, leading to different lower and upper bounds of $\overline{FM}_{CS}$.
Therefore, to quantify this sampling variability, we construct CIs for the lower and upper bounds of $\overline{FM}_{CS}$ using the bootstrap method \cite{Efron1979bootstrap}.

Let $A\subseteq X$, $\bar{L}_{\alpha,l,s}(A)$ be the CI for the lower bound and $\bar{U}_{\alpha,l,s}(A)$ be the CI for the upper bound.
They are formed as
\begin{align}
&\bar{L}_{\alpha,l,s}(A)=[L^{-}_{\alpha,l,s}(A),L^{+}_{\alpha,l,s}(A)], \notag \\
&\bar{U}_{\alpha,l,s}(A) = [U^{-}_{\alpha,l,s}(A),U^{+}_{\alpha,l,s}(A)], \notag 
\end{align}
where $\alpha$ represents the significant level ($1-\alpha$ represents the confidence level). Note that, if $|A|=1$ or $A=X$, it is unnecessary to construct the CIs.
We use the percentile method \cite{Bradley1987bootstrap} to construct two-sided bootstrap CIs, and the resample size of the bootstrap method is set as $1000$ as recommanded in \cite{Bradley1987bootstrap}.

Following the synthetic example, Table~\ref{The confidence intervals for the lower and upper bounds} shows the CIs for the lower and upper bounds, where $\bar{L}$ represents the lower bound and $\bar{U}$ represents the upper bound.

\begin{table}[h]
    \centering
    \caption{The confidence intervals for the lower and upper bounds of $\overline{FM}_{CS}$ based on the synthetic example}
    \begin{tabular}{c|c|ccc}
        \toprule
$\bar{g}_{l,s}$                               &           &   \boldmath{$\alpha=0.01$} & \boldmath{$\alpha=0.05$} & \boldmath{$\alpha=0.1$} \\ 
        \midrule
\multirow{2}{*}{$x_{1},x_{2}$}  & $\bar{L}$ &   $[0.671, 0.683]$     & $[0.671, 0.681]$  & $[0.671, 0.681]$ \\ 
                                & $\bar{U}$ &   $[0.715, 0.724]$     & $[0.719, 0.724]$  & $[0.719, 0.724]$ \\ 
        \midrule
\multirow{2}{*}{$x_{1},x_{3}$}  & $\bar{L}$ &   $[0.765, 0.782]$     & $[0.765, 0.781]$  & $[0.765, 0.781]$ \\ 
                                & $\bar{U}$ &   $[0.818, 0.836]$     & $[0.820, 0.829]$  & $[0.823, 0.829]$ \\ 
        \midrule
\multirow{2}{*}{$x_{2},x_{3}$}  & $\bar{L}$ &   $[0.869, 0.884]$     & $[0.869, 0.879]$  & $[0.869, 0.879]$ \\ 
                                & $\bar{U}$ &   $[0.918, 0.930]$     & $[0.918, 0.930]$  & $[0.920, 0.930]$ \\ 
        \bottomrule
    \end{tabular}
    \label{The confidence intervals for the lower and upper bounds}
\end{table}

\subsection{Identification of context-specific confidence interval-valued FM} 
\label{Identification of context-specific confidence interval-valued FM}

Let $\tilde{g}_{\alpha,l,s}$ be the $\overline{FM}_{CI}$, it is defined as

\begin{align} \label{IFM_CI}
\tilde{g}_{\alpha,l,s}(A) = 
\begin{cases}
\bigl[g(A),g(A)\bigr], & k=1 \\
\bigl[L^{-}_{\alpha,l,s}(A),U^{+}_{\alpha,l,s}(A)\bigr], & 1<k<I \\
\bigl[1,1\bigr], & k=I
\end{cases}, 
\end{align}
where $A\in A_{k}$. Since the $\overline{FM}_{CI}$ is constructed based on bootstrap method, in some extreme scenario, it could be possible that it breaks the monotonicity constraint. To address this, let $A\subset B\subset X$, if $L^{-}_{\alpha,l,s}(A) >L^{-}_{\alpha,l,s}(B)$, we set $L^{-}_{\alpha,l,s}(B)$ as $L^{-}_{\alpha,l,s}(A)$ to fix the monotonicity, and $U^{+}_{\alpha,l,s}(B)$ similarly.

\begin{theorem} \label{IFM-IE_theo_momotonicity}
$\overline{FM}_{CI}$ follows the monotonicity constraint.
\end{theorem}

Assume that $L_{l,s}$ and $U_{l,s}$ are the ground truth of lower and upper bounds of the $\overline{FM}_{CS}$. In other words, $\bar{L}_{\alpha,l,s}(A)$ is the CI for $L_{l,s}(A)$ (the lower bound of the $\overline{FM}_{CS}$), and $\bar{U}_{\alpha,l,s}(A)$ is the CI for $U_{l,s}(A)$ (the upper bound of the $\overline{FM}_{CS}$), where $A\in A_{k},\ k\in(1,I)$.

\begin{theorem} \label{theorem IFM-IE as CI}
Let $\tilde{g}_{\alpha,l,s}$ be a $\overline{FM}_{CI}$ and $\bar{g}_{l,s}$ be a $\overline{FM}_{CS}$. $\forall A\in A_{k},\ k\in(1,I)$, $\tilde{g}_{\alpha,l,s}(A)$ can be regarded as the confidence interval for $\bar{g}_{l,s}(A)$ with a confidence level of $1-\alpha+\frac{\alpha^{2}}{4}$ as
\begin{align}
P\left(\bar{g}_{l,s}(A)\subseteq \tilde{g}_{\alpha,l,s}(A)\right)=1-\alpha+\frac{\alpha^{2}}{4}. 
\end{align}
\end{theorem}

Typically, $\alpha$ is set as $0.01$, $0.05$ or $0.1$, then the confidence level of $\tilde{g}_{\alpha,l,s}(A)$ is $0.990025$, $0.959625$ or $0.9025$ respectively, which can be treated as $0.99$, $0.95$ or $0.90$. In other words, if $\alpha$ is set as $0.01$, $0.05$ or $0.1$, $\tilde{g}_{\alpha,l,s}(A)$ can be regarded as the $99\%$, $95\%$ or $90\%$ confidence interval for $\bar{g}_{l,s}(A)$.

Note that, the $\overline{FM}_{CI}$ is the CI for the $\overline{FM}_{CS}$. And, as discussed in the previous section, the $\overline{FM}_{CS}$ is structured based on top-performing FMs, where the `ideal' FM is more likely to be located in it. Building on these, the $\overline{FM}_{CI}$ is regarded as the confidence interval for the `ideal' FM at
the \textit{\textbf{same}} $1-\alpha$ confidence level.

Fig.~\ref{consrtuction demo} provides a visual demonstration for the construction of $\tilde{g}_{\alpha,l,s}(\{x_{2},x_{3}\})$ in the $\overline{FM}_{CI}$ together with all the mid-step intervals used in this approach. The points in green and blue (top-perform FMs and all FMs) are the same as shown in Fig.~\ref{all_g_0_to_1_g23}, but here we pick some of the points for visualization.
Algorithm~\ref{Obtain confidence intervals using bootstrap method} also provides the construction of the $\overline{FM}_{CI}$ including the bootstrap process. Note that $BT$ represents the bootstrap resample size and is set to $1000$ as discussed in the previous subsection.

\begin{figure}[h]
    \centering
    \includegraphics[width=0.45\textwidth]{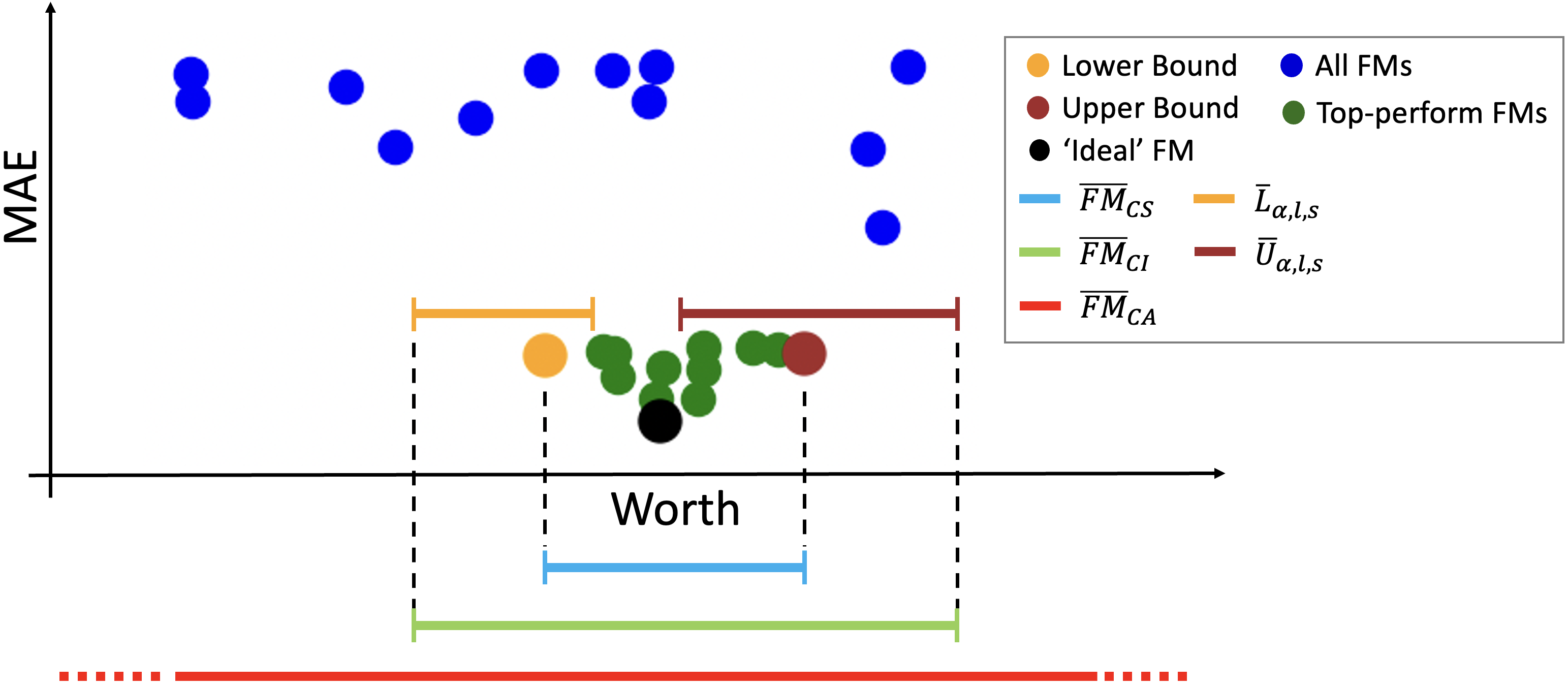}
    \caption{Visual demonstration for the construction of $\tilde{g}_{\alpha,l,s}(\{x_{2},x_{3}\})$ in the $\overline{FM}_{CI}$.}
    \label{consrtuction demo}
\end{figure}

\begin{algorithm}[h]
\caption{Construction of the $\overline{FM}_{CI}$}
\begin{algorithmic}[1]
\STATE Input the MC sample $g^{MC}_{s}$
\FOR {$k=1$ to $I$}
    \STATE For all $A\in A_{k}$
    \IF{k=1}
    \STATE $\tilde{g}_{\alpha,l,s}(A)=[g(A),g(A)]$
    \ELSIF{k=I}
    \STATE $\tilde{g}_{\alpha,l,s}(A)=[1,1]$
    \ELSE
        \FOR {$b = 1$ to $BT$}
            \STATE Randomly draw $s$ samples with replacement from $g^{MC}_{s}$ as $g^{MC,(b)}_{s}$
            \STATE Select the top-performing set as discussed in Section~\ref{subsec:Exploring the distribution of the FMs} from $g^{MC,(b)}_{s}$ as $\boldsymbol{g}^{(b)}_{l,s}$
            \FOR {$A\in A_{k}$}
                \STATE Select the minimum and maximum points of $\boldsymbol{g}^{(b)}_{l,s}(A)$ as $L^{(b)}_{l,s}(A)$ and $U^{(b)}_{l,s}(A)$, respectively.
            \ENDFOR
        \ENDFOR
        \FOR {$A\in A_{k}$}
            \STATE Let $L^{B}_{l,s}(A)$ and $U^{B}_{l,s}(A)$ be two sets, where $L^{B}_{l,s}(A)=\{L^{(b)}_{l,s}(A)|b=1,\cdots, BT\}$ and 
                   $U^{B}_{l,s}(A)=\{U^{(b)}_{l,s}(A)|b=1,\cdots, BT\}$.
            \STATE Calculate the $\frac{\alpha}{2}$ and $1-\frac{\alpha}{2}$ quantiles of $L^{B}_{l,s}(A)$ as $L^{-}_{\alpha,l,s}(A)$ and $L^{+}_{\alpha,l,s}(A)$
            \STATE Calculate the $\frac{\alpha}{2}$ and $1-\frac{\alpha}{2}$ quantiles of $U^{B}_{l,s}(A)$ as $U^{-}_{\alpha,l,s}(A)$ and $U^{+}_{\alpha,l,s}(A)$
            \STATE $\tilde{g}_{\alpha,l,s}(A)=[L^{-}_{\alpha,l,s}(A),U^{+}_{\alpha,l,s}(A)]$ 
        \ENDFOR
    \ENDIF
\ENDFOR
\end{algorithmic}
\label{Obtain confidence intervals using bootstrap method}
\end{algorithm}

Following the synthetic example, set $\alpha$ as $0.01$, $0.05$ and $0.1$, and the resulting $\overline{FM}_{CI}$ is shown is Table~\ref{The results for synthetic example}. The setting of this paper stands the same as Table~\ref{The IFM-CS for the synthetic example}. 
It seems that the $\overline{FM}_{CI}$ does not change significantly for different $\alpha$. 
This may be due to the choice of the threshold $l$, under which $L_{l,s}$ and $U_{l,s}$ appear to be stable.
The next subsection provides evidence that $\alpha$ does influence the $\overline{FM}_{CI}$. 

\begin{table}[h]
    \centering
    \caption{The $\overline{FM}_{CI}$ for the synthetic example}
    \begin{tabular}{c|ccc|c}
        \toprule
        \boldmath{$k=1$}        &  \multicolumn{3}{c|}{\boldmath{$k=2$}}         &   \boldmath{$k=3$} \\ 
        \midrule
        &  \boldmath{$\alpha=0.01$}  & \boldmath{$\alpha=0.05$}  & \boldmath{$\alpha=0.1$}  & \\
        \midrule
$0.4$   & $[0.671, 0.724]$     & $[0.671, 0.724]$  & $[0.671, 0.724]$ & $1$\\ 
$0.5$   & $[0.765, 0.836]$     & $[0.765, 0.829]$  & $[0.765, 0.829]$ & \\ 
$0.6$   & $[0.869, 0.930]$     & $[0.869, 0.930]$  & $[0.869, 0.930]$ & \\ 
        \bottomrule
    \end{tabular}
    \label{The results for synthetic example}
\end{table}

\subsection{Testing the confidence level of the $\overline{FM}_{CI}$} \label{test of confidence level for the FM}

As discussed above, $\overline{FM}_{CI}$ is the CI for $\overline{FM}_{CS}$ at the $1-\alpha$ confidence level. 
Although the $\overline{FM}_{CI}$ is regarded as the CI for the `ideal' FM, 
it does not guarantee that the actual confidence level of $\overline{FM}_{CI}$ for the `ideal' FM stays the same.
In this section, we conduct experiments to test whether the actual confidence level of the $\overline{FM}_{CI}$ for the `ideal' FM is consistent with $1-\alpha$. 
Since the threshold $l$ influences the length of the $\overline{FM}_{CI}$ and thus affects its confidence level for the `ideal' FM, 
the existence of a feasible $l$ that achieves a confidence level of $1-\alpha$ would support the claim that $\overline{FM}_{CI}$ serves as the CI for the `ideal' FM at the $1-\alpha$ confidence level.

Referring to the statement in \cite{hazra2017CI}, a interpretation of a $95\%$ CI is: `This implies that were the estimation process to be repeated over
and over with random samples from the same population, then $95\%$ of the calculated intervals would be expected to contain the true value.'.
So, the $\overline{FM}_{CI}$ for the `ideal' FM with significant level ($\alpha$) can be interpreted as: $(1-\alpha)\cdot 100\%$ of the $\overline{FM}_{CI}$ generated by resampling the MC sample ($g^{MC}_{s}$) encompass the `ideal' FM. 

Let $Gen$ be the generation of the regeneration of the $\overline{FM}_{CI}$, $\boldsymbol{l}$ be a set that contains multiple thresholds that are selected for testing. 
Algorithm~\ref{Calculation for the actual confidence level} shows the calculation of the actual confidence level of the $\overline{FM}_{CI}$ for the `ideal' FM.

\begin{algorithm}[h]
\caption{Calculation for the actual confidence level of the $\overline{FM}_{CI}$}
\begin{algorithmic}[1]
\STATE Let $\tilde{g}_{\alpha,l,s}$ be a $\overline{FM}_{CI}$
\FOR{$l\in\boldsymbol{l}$}
    \STATE Define $Count$ and set it as $0$
    \FOR{b = 1 to $Gen$} 
        \STATE Sample the MC sample as $g^{(b),MC}_{s}$
        \STATE Obtain the $\overline{FM}_{CI}$ according to Algorithm~\ref{Obtain confidence intervals using bootstrap method} as $\tilde{g}^{(b)}_{\alpha,l,s}$ 
        \IF{The `ideal' FM is in $\tilde{g}^{(b)}_{\alpha,l,s}$}
            \STATE $Count=Count+1$
        \ENDIF
    \ENDFOR
    \STATE The actual confidence level of $\tilde{g}_{\alpha,l,s}$ is $\frac{Count}{Gen}\cdot100\%$
\ENDFOR
\end{algorithmic}
\label{Calculation for the actual confidence level}
\end{algorithm}

Following the synthetic example, let $\boldsymbol{l}=\{0.1,0.05,0.01,0.008,0.006,0.004,0.002\}$, set $Gen = 100$, the actual confidence level of the $\overline{FM}_{CI}$ is shown in
Table~\ref{The actual confidence level of IFM for the synthetic example}, where $g12$ represents $g(\{x_{1},x_{2}\})$, $g13$ represents $g(\{x_{1},x_{3}\})$, and $g23$ represents $g(\{x_{2},x_{3}\})$. The percent sign ($\%$) is omitted for simplicity and the three values in each column correspond to the percentage to which the ideal measure is captured within the resulting confidence interval for each of the three source combinations $\{x_{1},x_{2}\}$, $\{x_{1},x_{3}\}$ and $\{x_{2},x_{3}\}$. 
The bold values in the table indicate that the corresponding $\overline{FM}_{CI}$ can be treated as the CI for the `ideal' FM at the same $1-\alpha$ confidence level. 
Note that we compute the average across the three source combinations, select the values closest to $1-\alpha$, and mark them in bold. 

\begin{table}[h]
    \centering
    \caption{The actual confidence level of the $\overline{FM}_{CI}$ for the synthetic example}
    \begin{tabular}{c|c|c|c}
        \toprule
$l$                     & \boldmath{$\alpha=0.01$}  & \boldmath{$\alpha=0.05$}  & \boldmath{$\alpha=0.1$}  \\          
        \midrule
                        & $g_{12},g_{13},g_{23}$ & $g_{12},g_{13},g_{23}$ & $g_{12},g_{13},g_{23}$ \\
        \midrule
1                       & $100,100,100$    & $100,100,100$    & $100,100,100$  \\
0.1                     & $100,100,100$    & $100,100,100$    & $100,100,100$  \\
0.05                    & $100,100,100$    & $100,100,100$    & $100,100,100$  \\
0.01                    & $100,100,100$    & $99,100,100$     & $98,100,100$   \\
0.008                   & $100,100,100$    & $98,100,100$     & $98,100,100$   \\
0.006                   & $98,100,100$     & $97,100,100$     & $95,97,99$   \\
0.004                   & $\mathbf{97,100,100}$ & $\mathbf{93, 95, 99}$ & $90,91,98$   \\
0.002                   & $92,93,98$       & $85,82,92$       & $75,72,78$     \\
        \bottomrule
    \end{tabular}

    \label{The actual confidence level of IFM for the synthetic example}
\end{table}

It can be implied that $\overline{FM}_{CI}$ can be regraded as the $99\%$ CI for the `ideal' FM when $l=0.004$. 
It is also possible to treat the $\overline{FM}_{CI}$ as the $95\%$ CI for the `ideal' FM when $l=0.004$.
However, none of the selected thresholds is feasible for constructing the $\overline{FM}_{CI}$ as the $90\%$ CI, which is limited by the selection of the threshold set $\boldsymbol{l}$ and the MC sample size $s$. Further discussion on the selection of hyperparameters, and the investigation of advanced approaches for constructing the $\overline{FM}_{CI}$, is left for future work. 
Note that the aim of this test is to empirically show that, by adjusting $l$, 
the $\overline{FM}_{CI}$ can be regraded as the confidence interval for the `ideal' FM at the $1-\alpha$ confidence level, rather than to choose the proper $l$. 
Therefore, we do not adjust Table~\ref{The results for synthetic example}, as it is intended to demonstrate the generation of the $\overline{FM}_{CI}$. One can re-generate the $\overline{FM}_{CI}$ by changing $l$ to $0.004$ and $\alpha$ to $0.01$ and $0.05$, so that it can be the actual confidence interval for the `ideal' FM.

As shown in this Table~\ref{The actual confidence level of IFM for the synthetic example}, 
when $l$ is $0.006$, $0.004$ or $0.002$, the actual confidence levels of the $\overline{FM}_{CI}$ differ across different $\alpha$. If $\alpha$ does not affect the $\overline{FM}_{CI}$, the actual confidence levels would remain the same for different $\alpha$. Therefore, although the $\overline{FM}_{CI}$ shown in Table~\ref{The IFM-CS for the synthetic example} are almost the same across different $\alpha$, $\alpha$ does affect the $\overline{FM}_{CI}$ when different threshold ($l$) is used.   
As $l$ affects the actual confidence level, further research is needed to determine an appropriate choice of $l$, and we leave this for future work.

Also note that, when $l=0.002$, there is only one discrete FM in the top-performing set, so the proposed approach acts as a calculation of the CI for a discrete FM, where this FM is identified by choosing the best discrete FM from the MC sample.

\subsection{The Choquet fuzzy integral based on the $\overline{FM}_{CI}$}

As the $\overline{FM}_{CI}$ is obtained, it is natural to fuse information using the Choquet FI (CFI). Since the $\overline{FM}_{CI}$ is an interval-valued FM, the CFI output in respect to it is also an interval. 
We argued that the CFI output based on the $\overline{FM}_{CI}$ can be regarded as the confidence interval for the real information fusion result (the ground truth) under the same ($1-\alpha$) confidence level. 
In other words, when we fuse information using the CFI based on multiple regenerated $\overline{FM}_{CI}$ with significant level $\alpha$, 
$(1-\alpha)\cdot100\%$ of the CFI outputs (confidence intervals) cover the real information fusion result (ground truth).

Fig.~\ref{The CFI output for the $205^{th}$ sample in the synthetic example} demonstrates the CFI output of the $205^{th}$ sample in the synthetic example together with the CFI output in respect to four discrete FM generation approaches, where the $205^{th}$ sample is $(0, 2.22, 5.56)$,
GA-FM represents the discrete FM obtained by genetic algorithm \cite{Scott2017ChI-DE} but based on the given densities. 

\begin{figure}[h]
    \centering
    \includegraphics[width=0.45\textwidth]{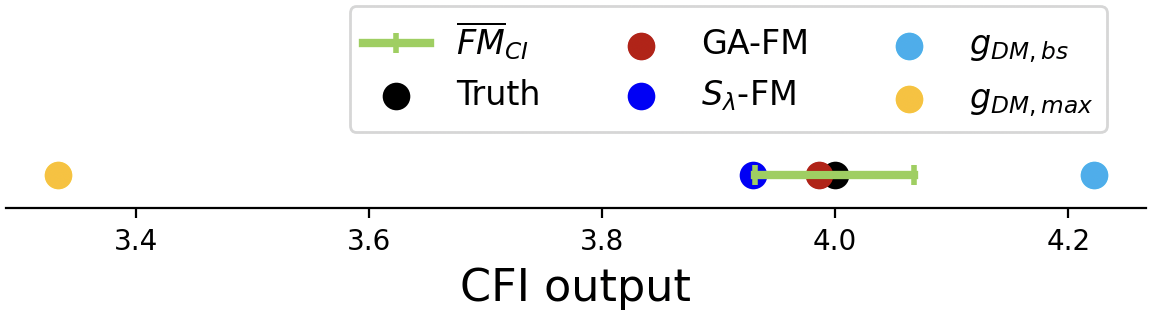}
    \caption{The CFI output for the $205^{th}$ sample in the synthetic example.}
    \label{The CFI output for the $205^{th}$ sample in the synthetic example}
\end{figure}

\subsection{Testing the confidence level of the Choquet FI outputs based on the $\overline{FM}_{CI}$} \label{test based on CFI}

A major issue for the test in Section~\ref{test of confidence level for the FM} is that the `ideal' FM is unknown in the real-world scenario.
Similarly as section~\ref{test of confidence level for the FM}, we want to test whether the actual confidence level of the CFI outputs is consistent with $1-\alpha$.

Here the testing process is similar to Algorithm~\ref{Calculation for the actual confidence level} and is shown as Algorithm~\ref{Calculation for the actual confidence level for the CFI}.
Unlike the `ideal' FM, which is unknown, it is a guarantee that the ground truth is known and thus the actual confidence level of the CFI output based on the $\overline{FM}_{CI}$ is obtainable. 

\begin{algorithm}[h]
\caption{Calculation for the actual confidence level of the CFI output based on the $\overline{FM}_{CI}$}
\begin{algorithmic}[1]
\STATE Let $\tilde{g}_{\alpha,l,s}$ be a $\overline{FM}_{CI}$
\FOR{$l\in\boldsymbol{l}$}
    \STATE Define $Count$ and set it as $0$
    \FOR{b = 1 to $Gen$} 
        \STATE Sample the MC sample as $g^{(b),MC}_{s}$
        \STATE Obtain the $\overline{FM}_{CI}$ according to Algorithm~\ref{Obtain confidence intervals using bootstrap method} as $\tilde{g}^{(b)}_{\alpha,l,s}$ 
        \FOR{t = 1 to $T$}
            \STATE Calculate the CFI based on $\tilde{g}^{(b)}_{\alpha,l,s}$ as $C_{\tilde{g}^{(b)}_{\alpha,l,s}}(h(X,t))$
            \IF{The ground truth ($y_{t}$) is in $C_{\tilde{g}^{(b)}_{\alpha,l,s}}(h(X,t))$}
                \STATE $Count=Count+1$
            \ENDIF
        \ENDFOR
    \ENDFOR
    \STATE The actual confidence level of the CFI output based on $\tilde{g}_{\alpha,l,s}$ is $\frac{Count}{Gen\cdot T}\cdot100\%$
\ENDFOR
\end{algorithmic}
\label{Calculation for the actual confidence level for the CFI}
\end{algorithm}

Following the synthetic example, Table~\ref{The actual confidence level of CFI for the synthetic example} shows the actual confidence level of the CFI output for the ground truth,
which follows the settings as Table~\ref{The actual confidence level of IFM for the synthetic example}.  
Note that this actual confidence level is an average across the dataset. 

\begin{table}[h]
    \centering
    \caption{The actual confidence level of the CFI output for the synthetic example}
    \begin{tabular}{c|c|c|c}
        \toprule
$l$                     & \boldmath{$\alpha=0.01$}  & \boldmath{$\alpha=0.05$}  & \boldmath{$\alpha=0.1$}  \\          
        \midrule                    
1                       & $100$    & $100$    & $100$  \\
0.1                     & $100$    & $100$    & $100$  \\
0.05                    & $100$    & $100$    & $100$  \\
0.01                    & $100$    & $99.72$  & $99.43$   \\
0.008                   & $100$    & $99.43$  & $99.43$   \\
0.006                   & $99.43$  & $99.15$  & $97.44$   \\
0.004                   & \textbf{99.15}  & \textbf{96.30}  & $94.02$   \\
0.002                   & $95.16$  & $88.32$  & $78.63$     \\
        \bottomrule
    \end{tabular}
    \label{The actual confidence level of CFI for the synthetic example}
\end{table}

Comparing these two tables, the actual confidence levels of the $\overline{FM}_{CI}$ are $100\%$ across all three worths, it is aligned that the actual confidence levels of the CFI output are also $100\%$.
Also when $l=0.004$, $\alpha=0.01,0.05$, the resulting $\overline{FM}_{CI}$ act good in the tests, which also fit that the CFI output based on these two FMs act good in the tests.
From an overall perspective, the actual confidence levels of the $\overline{FM}_{CI}$ and the actual confidence levels of the CFI output are associated: when one increases, the other also increases; when one decreases, the other also decreases.

Building on these, we argue that, in practice, we can use the actual confidence level of the CFI output to represent the actual confidence level of the $\overline{FM}_{CI}$.
In other words, if we want to know whether $1-\alpha$ is the `real' confidence level of the $\overline{FM}_{CI}$, we can test the actual confidence level of the CFI output to see whether it is aligned with $1-\alpha$.

\section{Experiment} \label{Experiment}

\subsection{Data description}

In addition to the synthetic example, another example from Li et al. \cite{Li2013example} is used to demonstrate the proposed approaches shown in Table~\ref{Example 17}. 
This is a customer evaluation problem (15 customers) where the left column presents the rating of four attributes and the right column represents the overall evaluation. The four attributes are considered as the evidence, while the overall evaluation is regarded as the ground truth.

\begin{table}[h]
    \centering
    \caption{Example $17$ in Li et al. \cite{Li2013example}}
    \begin{tabular}{cccc|c}
        \toprule
        \makecell{Transfer \\ time (T)} & 
        \makecell{Congestion \\ (C)} & 
        \makecell{Travel \\ time (TT)} & 
        \makecell{Ticket \\ price (TP)} & 
        \makecell{Evaluation \\ ($E_{i}$)} \\
        \midrule
                   $0$   & $0$   & $0.3$ & $1$   & $0.2$  \\
                   $0.3$ & $0.8$ & $0.5$ & $0.6$ & $0.5$  \\
                   $0.7$ & $0.9$ & $1$   & $0.7$ & $0.8$  \\
                   $0.1$ & $0$   & $0.4$ & $0.3$ & $0.15$ \\
                   $0.5$ & $0.4$ & $0.5$ & $0.6$ & $0.5$  \\
                   $1$   & $0.3$ & $0.8$ & $0.8$ & $0.7$  \\
                   $0.3$ & $1$   & $0$   & $1$   & $0.5$  \\
                   $0.9$ & $0.6$ & $0.5$ & $0.7$ & $0.7$  \\
                   \textbf{0.6}  & \textbf{0.8} & \textbf{0.2} & \textbf{0.4} & \textbf{0.5}  \\
                   $0.4$ & $1$   & $0.5$ & $0$   & $0.4$  \\
                   $0.8$ & $0.4$ & $1$   & $0.3$ & $0.6$  \\
                   $1$   & $0.8$ & $0.7$ & $0.5$ & $0.75$ \\
                   $0.2$ & $0.6$ & $1$   & $0.7$ & $0.5$  \\
                   $0$   & $0.3$ & $0.2$ & $0.9$ & $0.2$  \\
                   $0.4$ & $0.2$ & $0.8$ & $0$   & $0.3$  \\
        \bottomrule
    \end{tabular}
    \label{Example 17}
\end{table}

\subsection{Generating the context-agnostic interval-valued FM}

To generate the $\overline{FM}_{CA}$, the densities must be determined. This paper uses the densities from \cite{Vicen2022example} which 
are obtained by genetic algorithms as $g(\{T\}) = 0.2743$, $g(\{C\}) = 0.1946$, $g(\{TT\}) = 0.1769$ and $g(\{TP\}) = 0.1327$.
The $\overline{FM}_{CA}$ can be generated as discussed in Section~\ref{Generate interval-valued FM from densities} and is shown in Table~\ref{IFM-CA for Example 17}, where T,C stands for the combination $\{T,C\}$ and others similarly.

\begin{table}[h]
    \centering
    \caption{The $\overline{FM}_{CA}$ for Example $17$}
    \begin{tabular}{llll}
        \toprule
                   \boldmath{$k=1$} & \boldmath{$k=2$} & \boldmath{$k=3$} & \boldmath{$k=4$} \\
        \midrule
                    T: 0.2743          &T,C: [0.2743,1]          &T,C,TT: [0.2743,1]          & X: 1  \\ 
                    C: 0.1946          &T,TT: [0.2743,1]         &T,C,TP: [0.2743,1]          &        \\ 
                    TT: 0.1769         &T,TP: [0.2743,1]         &T,TT,TP: [0.2743,1]         &         \\
                    TP: 0.1327         &C,TT: [0.1946,1]         &C,TT,TP: [0.1946,1]         &          \\      
                                       &C,TP: [0.1946,1]         &                               &           \\
                                       &TT,TP: [0.1769,1]        &                               &            \\
        \bottomrule
    \end{tabular}
    \label{IFM-CA for Example 17}
\end{table}

\subsection{Generating the context-specific interval-valued FM}

The Choquet FI (CFI) is chosen to construct $\overline{FM}_{CS}$ as the densities from \cite{Vicen2022example} are generated based on the CFI.
Set $l=0.01$, $s=100000$, Table~\ref{IFM-CS for Example 17} shows the $\overline{FM}_{CS}$ for Example $17$. Note that this table follows the same setting as Table~\ref{IFM-CA for Example 17} and the captions of the combinations, e.g. T,C:, are removed for simplicity.

\begin{table}[h]
    \centering
    \caption{The $\overline{FM}_{CS}$ for Example $17$}
    \begin{tabular}{cccc}
        \toprule
                   \boldmath{$k=1$} & \boldmath{$k=2$} & \boldmath{$k=3$} & \boldmath{$k=4$} \\
        \midrule
                    0.2743          &[0.3866,0.6853]          &[0.5372,0.7457]          & 1  \\ 
                    0.1946          &[0.4149,0.5958]         &[0.6821,0.9643]          &        \\ 
                    0.1769         &[0.3342,0.6817]         &[0.6550,0.8347]         &         \\
                    0.1327         &[0.3284,0.4787]         &[0.4876,0.6433]         &          \\      
                                       &[0.2269,0.4621]         &                               &           \\
                                       &[0.2009,0.4184]        &                               &            \\
        \bottomrule
    \end{tabular}
    \label{IFM-CS for Example 17}
\end{table}

\subsection{Generating the context-specific confidence interval-valued FM}

Going one step further, significant level $\alpha$ is added to generate the $\overline{FM}_{CI}$. 
Table~\ref{IFM-CI for Example 17} shows the $\overline{FM}_{CI}$ when $\alpha$ is set as $0.01$, $0.05$ and $0.1$.

\begin{table}[h]
    \centering
    \caption{The $\overline{FM}_{CI}$ for Example $17$}
    \begin{tabular}{c|c|c|c}
        \toprule
$\tilde{g}_{\alpha,l,s}$                      &  \boldmath{$\alpha=0.01$}  & \boldmath{$\alpha=0.05$}  & \boldmath{$\alpha=0.1$} \\
        \midrule
T,C                   &  [0.3424,0.6853]        &[0.3424,0.6853]         &[0.3866,0.6853]           \\ 
T,TT                  &  [0.4149,0.5958]        &[0.4149,0.5958]         &[0.4149,0.5958]           \\ 
T,TP                  &  [0.3342,0.7131]        &[0.3342,0.7131]         &[0.3342,0.7131]          \\
C,TT                  &  [0.2329,0.5153]        &[0.2329,0.5153]         &[0.2329,0.4787]          \\      
C,TP                  &  [0.2269,0.5148]        &[0.2269,0.5148]         &[0.2269,0.5148]                                 \\
TT,TP                 &  [0.2009,0.5100]        &[0.2009,0.5100]         &[0.2009,0.4184]                                \\
        \midrule
T,C,TT                &  [0.5372,0.8018]        &[0.5372,0.8018]         &[0.5372,0.7457]                                 \\
T,C,TP                &  [0.6821,0.9755]        &[0.6821,0.9755]         &[0.6821,0.9723]                                 \\
T,TT,TP               &  [0.6283,0.8347]        &[0.6283,0.8347]         &[0.6283,0.8347]                                 \\
C,TT,TP               &  [0.4437,0.6484]        &[0.4437,0.6433]         &[0.4437,0.6433]                                \\                                   
                                   
        \bottomrule
    \end{tabular}
    \label{IFM-CI for Example 17}
\end{table}

\subsection{Obtaining the overall evaluation using the CFI based on the $\overline{FM}_{CI}$} \label{CFI output real experiment}

As the $\overline{FM}_{CI}$ is obtained, the next step is to fuse the four attributes to make the overall evaluation using the CFI.
Here, the CFI output based on the $\overline{FM}_{CI}$ is regraded as the $1-\alpha$ confidence interval for the overall evaluation. 

Following the previous settings and set $\alpha$ as $0.05$, the CFI output is $[0.4438, 0.5711]$ and is
shown in Fig.~\ref{The CFI output for the $9^{th}$ sample in Example $17$}.
We choose the $9^{th}$ sample, which is highlighted in bold in Table~\ref{Example 17}, for demonstration.
Note that (Max,+) stands for the (MAX,$\oplus$)-transforms obtained from \cite{Vicen2022example}.

\begin{figure}[h]
    \centering
    \includegraphics[width=0.45\textwidth]{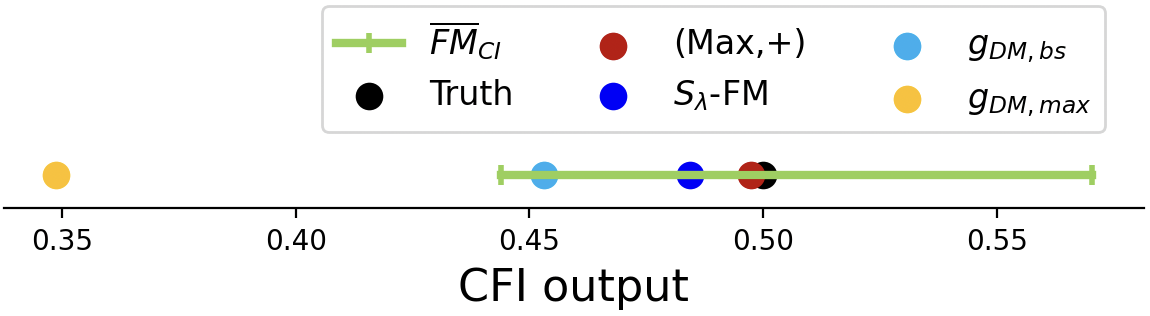}
    \caption{The CFI output for the $9^{th}$ sample in Example $17$.}
    \label{The CFI output for the $9^{th}$ sample in Example $17$}
\end{figure}

\subsection{Testing the confidence level} \label{real test CI}

Here, since the `ideal' FM is unknown, as discussed in Section~\ref{test based on CFI},
we use the confidence level of the CFI output to represent the confidence level of the $\overline{FM}_{CI}$.
Table~\ref{The actual confidence level for example 17} shows the actual confidence level of the CFI output.
These results indicate that the $\overline{FM}_{CI}$ can be regarded as the $99\%$ CI for the `ideal' FM when $l=0.05$. 
Similarly, the $\overline{FM}_{CI}$ can be regarded as the $95\%$ and the $90\%$ CI for the `ideal' FM when $l=0.01$ and $l=0.005$, respectively. 

\begin{table}[h]
    \centering
    \caption{The actual confidence level of the CFI output for Example $17$}
    \begin{tabular}{c|c|c|c}
        \toprule
$l$                     & \boldmath{$\alpha=0.01$}  & \boldmath{$\alpha=0.05$}  & \boldmath{$\alpha=0.1$}  \\          
        \midrule                    
1                       & $100$    & $100$    & $100$  \\
0.1                     & $100$    & $100$    & $100$  \\
0.05                    & \textbf{99.93}  & $99.93$  & $99.87$  \\
0.01                    & $95.33$  & \textbf{93.87}  & $93.34$   \\
0.005                   & $92.07$   & $91.00$  & \textbf{90.20}   \\
0.004                   & $91.27$   & $89.73$  & $88.53$   \\
0.003                   & $90.27$  & $87.80$  & $85.60$   \\
0.002                   & $88.20$  & $84.27$  & $81.93$     \\
0.001                   & $83.13$  & $77.33$  & $70.07$     \\
        \bottomrule
    \end{tabular}
    \label{The actual confidence level for example 17}
\end{table}

As discussed in Section~\ref{test of confidence level for the FM}, this test empirically shows the existence of $l$ that can make the $\overline{FM}_{CI}$ be the confidence interval for the `ideal' FM at $1-\alpha$ confidence level. One can regenerate the $\overline{FM}_{CI}$ by adjusting $l$ to $0.05$, $0.01$, and $0.005$, and $\alpha$ to $0.01$, $0.05$, and $0.1$, respectively. Although we do not provide all the regenerated $\overline{FM}_{CI}$ here, under the setting of $l=0.01$ and $\alpha=0.05$, the third column in Table~\ref{IFM-CI for Example 17} is the CI for the `ideal' FM at the $0.95$ confidence level. In addition, the CFI output based on this $\overline{FM}_{CI}$ is the CI for the ground truth at the $0.95$ confidence level.
For example, as calculated in Section~\ref{CFI output real experiment}, $[0.4438, 0.5711]$ is the $95\%$ CI for the ground truth of the $9^{th}$ sample.

\section{Conclusions \& Limitations} \label{Conclusion}

This paper discusses how Fuzzy Measure generation for information aggregation. Many approaches generate discrete FMs by leveraging the densities and the monotonicity constraint. However, we show that these information is insufficient to uniquely identify a discrete FM, but a context-agnostic interval-valued FM ($\overline{FM}_{CA}$) can be obtained. We demonstrate the structure of this FM and show that density-based FM generation approaches place FMs within it but sometimes in quite different places. 

We further show that once a dataset and a specific FI are given, a more specific interval-valued FM can be determined. A Monte Carlo-based approach is proposed to generate such an FM, ($\overline{FM}_{CS}$).

Going a step further, we describe how to transform $\overline{FM}_{CS}$ to a confidence level interval-valued FM ($\overline{FM}_{CI}$), providing a measure of its expected quality. This FM can be regarded as the confidence interval for the `ideal' FM.

Finally, we show empirically how the output of a Choquet FI based on the latter FM can be regarded as the confidence interval for the `ideal' information fusion result with the same confidence level, providing the first such a priori characterisation to the best of our knowledge.

A synthetic example for three individual sources is used alongside the introduction of the proposed approaches for illustration and demonstration. 
Another example of four individual sources drawn from the literature \cite{Li2013example} is also included to further demonstrate the proposed approaches.

One of the limitations of this paper is that the densities are considered to be crisp values. While this is the norm, different approaches, especially those based on optimization may obtain different densities, leading to interval-valued or maybe fuzzy set-valued densities. This is an area of work we expect to explore in the future. 
Another limitation is that the Monte Carlo process is computationally expensive, thus reducing computational complexity is worth exploring in future work, as is the application of the proposed approach within ensemble classification.

In addition, although the tests in Section~\ref{test of confidence level for the FM}, Section~\ref{test based on CFI}, and Section~\ref{real test CI} have empirically shown that $\overline{FM}_{CI}$ can be regarded as the confidence interval for the `ideal' FM at the $1-\alpha$ confidence level, it is not guaranteed that this is always the case.

\bibliographystyle{IEEEtran}
\bibliography{Manuscript}

\appendices
\section{Proof of Theorem~\ref{theorem:FM_CA_monotonocity}}

\begin{proof}
Let $A \subseteq B \subseteq X$,
\[
\max_{x_{i}\in B} g(\{x_{i}\})
   = \max\!\left(
       \max_{x_{i}\in A} g(\{x_{i}\}),
       \max_{x_{i}\in B\setminus A} g(\{x_{i}\})
     \right).
\]

If 
\[
\max_{x_{i}\in A} g(\{x_{i}\})
   \geq \max_{x_{i}\in B\setminus A} g(\{x_{i}\}),
\]
then
\[
\max_{x_{i}\in B} g(\{x_{i}\})
   = \max_{x_{i}\in A} g(\{x_{i}\}).
\]

If 
\[
\max_{x_{i}\in A} g(\{x_{i}\})
   < \max_{x_{i}\in B\setminus A} g(\{x_{i}\}),
\]
then
\[
\max_{x_{i}\in B} g(\{x_{i}\})
   = \max_{x_{i}\in B\setminus A} g(\{x_{i}\})
   > \max_{x_{i}\in A} g(\{x_{i}\}).
\]

Therefore,
\[
\max_{x_{i}\in B} g(\{x_{i}\})
   \geq \max_{x_{i}\in A} g(\{x_{i}\}).
\]

It is obvious that
\[
1 \geq \max_{x_{i}\in A} g(\{x_{i}\}).
\]

According to (\ref{interval_larger}),
\[
\bigl[\max_{x_{i}\in A} g(\{x_{i}\}),\,\max_{x_{i}\in A} g(\{x_{i}\})\bigr]
   \leq \bigl[\max_{x_{i}\in B} g(\{x_{i}\}),\,1\bigr],
\]
\[
\bigl[\max_{x_{i}\in A} g(\{x_{i}\}),\,1\bigr]
   \leq \bigl[\max_{x_{i}\in B} g(\{x_{i}\}),\,1\bigr],
\]
\[
\bigl[\max_{x_{i}\in A} g(\{x_{i}\}),\,1\bigr]
   \leq \bigl[1,\,1\bigr].
\]

According to (\ref{IFM_DM}),
\[
\bar{g}_{CA}(A) \leq \bar{g}_{CA}(B).
\]
\end{proof}

\section{Proof of Theorem~\ref{IFM-IE_theo_momotonicity}}

\begin{proof}
Monotonicity is manually fixed as discussed in Section~\ref{Identification of context-specific confidence interval-valued FM}.
\end{proof}

\section{Proof of Theorem~\ref{theorem IFM-IE as CI}}

\begin{proof}
Let $k\in(1,I)$, $\forall\ A\in A_{k} $,
\begin{align}
\tilde{g}_{\alpha,l,s}(A)&=[L^{-}_{\alpha,l,s}(A),U^{+}_{\alpha,l,s}(A)], \notag \\
\bar{g}_{l,s}(A) &= [L_{l,s},U_{l,s}]. \notag
\end{align}

Since $\bar{L}_{\alpha,l,s}(A)$ and $\bar{U}_{\alpha,l,s}(A)$ are two-sided bootstrap $1-\alpha$ CIs for $\bar{g}_{l,s}(A)$, then
\begin{align}
&P\left(L^{-}_{\alpha,l,s}(A)>L_{l,s}(A)\right)=\frac{\alpha}{2}, \notag \\ 
&P\left(U^{+}_{\alpha,l,s}(A)<U_{l,s}(A)\right)=\frac{\alpha}{2}.\notag
\end{align}

Since
\begin{align}
&\bar{g}_{l,s}(A)\subseteq \tilde{g}_{\alpha,l,s}(A) \notag \\
&\iff L^{-}_{\alpha,l,s}(A)\leq L_{l,s}(A),\ U^{+}_{\alpha,l,s}(A)\geq U_{l,s}(A). \notag
\end{align}
, then 
\begin{align}
&P\left(\bar{g}_{l,s}(A)\subseteq \tilde{g}_{\alpha,l,s}(A)\right) \notag \\
&=P(L^{-}_{\alpha,l,s}(A)\leq L_{l,s}(A))\cdot P(U^{+}_{\alpha,l,s}(A)\geq U_{l,s}(A)). \notag
\end{align}

Since
\begin{align}
P(L^{-}_{\alpha,l,s}(A)\leq L_{l,s}(A))=\bigl(1-P(L^{-}_{\alpha,l,s}(A)>L_{l,s}(A))\bigr) \notag \\ 
P(U^{+}_{\alpha,l,s}(A)\geq U_{l,s}(A))=\bigl(1-P(U^{+}_{\alpha,l,s}(A)<U_{l,s}(A))\bigr), \notag
\end{align}
we have
\begin{align}
&P\left(\bar{g}_{l,s}(A)\subseteq \tilde{g}_{\alpha,l,s}(A)\right) 
= (1-\frac{\alpha}{2})\cdot (1-\frac{\alpha}{2}) = 1-\alpha+\frac{\alpha^{2}}{4}. \notag
\end{align}
\end{proof}

\end{document}